\documentclass{article}

\usepackage[utf8]{inputenc}  
\usepackage[T1]{fontenc}
\usepackage{microtype}
\usepackage{graphicx}
\usepackage{booktabs}
\usepackage{amsmath,amssymb,amsthm}
\usepackage{mathtools}
\usepackage{lmodern}
\usepackage{xspace}
\usepackage{float}
\usepackage{caption}

\usepackage[dvipsnames,table]{xcolor} 
\usepackage{multirow}
\usepackage{makecell}
\usepackage{xltabular}
\usepackage{rotating}
\usepackage{pifont}
\usepackage{threeparttablex}

\definecolor{navy}{RGB}{0,0,110}
\definecolor{boxfill}{RGB}{242,245,255}
\definecolor{deepskyblue}{rgb}{0.01171875, 0.3046875, 0.6328125}
\definecolor{softskyblue}{rgb}{0.53, 0.81, 0.92}
\definecolor{babyblue}{rgb}{0.68, 0.85, 0.90}
\definecolor{paleblue}{rgb}{0.85, 0.93, 1.0}

\newcommand{\cmark}{\textcolor{Green}{\ding{51}}}  
\newcommand{\xmark}{\textcolor{Red}{\ding{55}}}    

\usepackage{hyperref}
\usepackage[capitalise]{cleveref}

\usepackage[accepted]{icml2026}   

\usepackage{algorithm}
\usepackage{algorithmic}
\usepackage{listings}
\usepackage{listingsutf8}

\usepackage{subcaption} 
\usepackage{svg}
\usepackage{tikz}
\usepackage{overpic}
\usepackage{pgf} 
\usepackage{wrapfig}
\usepackage{varwidth}
\usepackage[most]{tcolorbox}
\tcbuselibrary{listings,breakable,skins}
\usepackage{subfiles}
\usepackage{verbatim}
\usepackage{xcolor}

\usepackage{siunitx} 
\usepackage{enumitem}
\usepackage{soul}
\sethlcolor{yellow}
\usepackage[textsize=tiny]{todonotes}

\soulregister\autoref7
\soulregister\autoref7

\newcommand{\name}{\textsc{NanoQuant}\xspace}

\newcommand{\eg}{\textit{e.g.,}\xspace}

\newif\ifshowfixme \showfixmetrue   
\ifshowfixme

  \newenvironment{fixmeblock}{\begingroup\color{red}}{\endgroup}  
\else

\fi

\newif\ifshowrevised \showrevisedtrue   
\ifshowrevised
  \newenvironment{revisedblock}{\begingroup\color{blue}}{\endgroup}  
\else  
    
\fi

\newtheorem{theorem}{Theorem}
\newtheorem{proposition}{Proposition}
\newtheorem{lemma}{Lemma}

\theoremstyle{plain}
\newtheorem{corollary}[theorem]{Corollary}
\theoremstyle{definition}

\theoremstyle{remark}

\icmltitlerunning{\textsc{NanoQuant}: Efficient Sub-1-Bit Quantization of Large Language Models}

\begin{document}

\twocolumn[
  \icmltitle{\name: Efficient Sub-1-Bit Quantization of Large Language Models}

  \icmlsetsymbol{equal}{*}
  \icmlsetsymbol{corr}{\dag}
  
  \begin{icmlauthorlist}
    \icmlauthor{Hyochan Chong}{equal,sr}
    \icmlauthor{Dongkyu Kim}{equal,corr,sr}
    \icmlauthor{Changdong Kim}{sr}
    \icmlauthor{Minseop Choi}{sr}
  \end{icmlauthorlist}

    \icmlaffiliation{sr}{Samsung Research, Seoul, Korea}
    
    \icmlcorrespondingauthor{Dongkyu Kim}{dongkyu.k@samsung.com}
    
    \icmlkeywords{Large Language Models, Quantization, Binary Networks, ADMM, Model Compression, On-Device Inference}
    
    \vskip 0.3in
]

\printAffiliationsAndNotice{\icmlEqualContribution}

\begin{abstract}
Weight-only quantization has become a standard approach for efficiently serving large language models (LLMs).
However, existing methods fail to efficiently compress models to binary (1-bit) levels, as they either require large amounts of data and compute or incur additional storage.
In this work, we propose \name, a post-training quantization (PTQ) method to compress LLMs to both binary and \textit{sub-1-bit} levels.
\name formulates quantization as a low-rank binary factorization problem, and compresses full-precision weights to low-rank binary matrices and scales.
Specifically, it utilizes an efficient alternating direction method of multipliers (ADMM) solver to precisely initialize latent binary matrices and scales, and then tunes the initialized parameters through a block and model reconstruction process.
Consequently, \name establishes a new Pareto frontier in low-memory post-training quantization, and enables sub-1-bit compression.
\name makes large-scale deployment feasible on consumer hardware.
For example, it compresses Llama-2-70B by 24$\times$ in just 13 hours on a \textit{single} H100, enabling a 70B model to operate on a consumer 8 GB GPU.
Code is available at {\hypersetup{urlcolor=black}\href{https://github.com/SamsungLabs/NanoQuant}{\texttt{github.com/SamsungLabs/NanoQuant}}}.
\end{abstract}

\section{Introduction}
Large language models (LLMs) have demonstrated remarkable performance across a wide variety of tasks.
However, their extremely large size makes deployment costly.
Weight-only quantization offers a standard route to alleviate these bottlenecks~\cite{gptq, awq, omniquant, spinquant}. This has led to its widespread adoption within production-grade inference engines, such as vLLM \cite{paged_attention} and SGLang \cite{sglang}.

\begin{table}[t]
    \scriptsize
    \centering
    \caption{
    Comparison of LLM quantization frameworks.
    Methods are categorized by quantization scheme (PTQ vs. QAT), scalability to 70B+ models, and sub-1-bit capability.
    Only \name enables sub-1-bit compression among baselines.
    }
    \vspace{-2pt}
    \tabcolsep=0.1cm
    \setlength{\tabcolsep}{3pt}
    \renewcommand{\arraystretch}{1.00}
    \begin{tabular}{@{}lccccc@{}}
        \toprule

        \multirow{2}{*}[-3pt]{\makecell[c]{\textbf{Quantization Method}}} & 
        \multirow{2}{*}[-3pt]{\makecell[c]{\textbf{Scheme}}}
        & \multicolumn{3}{c}{\textbf{Compression}} \\

        \cmidrule(lr){3-5} 

        &  & \makecell[c]{\textbf{70B+ LLMs}} & \makecell[c]{\textbf{1-Bit}} & \makecell[c]{\textbf{Sub-1-Bit}}
         \\
        \midrule
        \makecell[l]{BiLLM \cite{bi_llm}} & PTQ & \cmark & \xmark & \xmark \\
        \makecell[l]{STBLLM \cite{stb_llm}} & PTQ & \cmark & \xmark & \xmark \\
        \makecell[l]{ARB-LLM \cite{arb_llm}} & PTQ & \cmark & \xmark & \xmark \\
        \makecell[l]{HB-LLM \cite{hb_llm}} & PTQ & \cmark & \xmark & \xmark \\
        \makecell[l]{OneBit \cite{onebit}} & QAT & \xmark & \cmark & \xmark \\
        \makecell[l]{BinaryMoS \cite{binarymos}} & QAT & \xmark & \cmark & \xmark \\
        \makecell[l]{DBF \cite{dbf}} & QAT & \xmark & \cmark & \xmark \\
        \makecell[l]{ParetoQ \cite{paretoq}} & QAT & \xmark & \cmark & \xmark \\
        \makecell[l]{LittleBit \cite{littlebit}} & QAT & \xmark & \cmark & \cmark \\
        \midrule
        \rowcolor{babyblue!40}
        \textbf{\name (Ours)} & PTQ & \cmark & \cmark & \cmark \\
        \bottomrule
    \end{tabular}
    \vspace{-10px}
\end{table}

Recent post-training quantization (PTQ) efforts have successfully pushed weight compression toward 2-bit~\cite{quip, qtip} and even 1-bit~\cite{bi_llm, arb_llm, hb_llm}.
However, breaking the \textit{sub-1-bit} barrier remains a challenge for current PTQ frameworks for two distinct reasons.
First, current binary PTQ methods utilize in-place binarization with full-precision scales (\eg $\mathbf{W} \approx \alpha \mathbf{B}_{\pm 1}$), an approach that is structurally bounded by a minimum of 1 bit per parameter.
Moreover, these techniques require complex weight-grouping metadata~\cite{bi_llm, stb_llm, ptq1p61, hb_llm}, causing effective bitrates to exceed 2 and even 3 bits \cite{stb_llm}.
Thus, a key challenge for sub-1-bit PTQ is to efficiently represent model parameters to overcome both the structural and storage limitations of current methods.

\begin{figure*}[!t]
    \centering
    \includegraphics[width=0.95\linewidth]{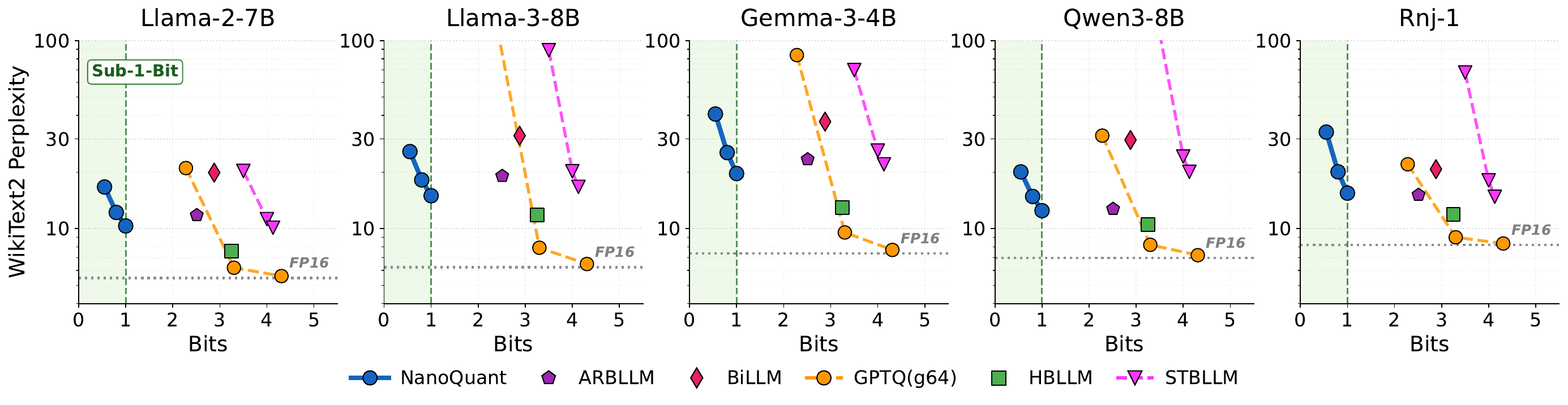}
    \caption{
    Perplexity comparison on WikiText-2.
    \name reaches 1-bit and sub-1-bit effective storage while remaining competitive with existing binary PTQ baselines that require substantially higher effective BPW.}
    \label{fig1_ppl}
\end{figure*}

In contrast, binary quantization-aware training (QAT) methods successfully compress LLMs to binary (1-bit) and even sub-1-bit levels with low-rank representations~\cite{littlebit,dbf}, overcoming both the structural and additional storage limitations inherent in binary PTQ methods.
Through an end-to-end training process, such binary QAT methods replace linear layer weights with compact, low-rank binary matrices and scales.
However, unlike PTQ methods, these QAT methods require hundreds of millions or billions of tokens, and utilize multiple GPUs over multiple days.
These demands are impractical for resource-constrained environments and limit such QAT methods from compressing larger, 70B parameter models.
Therefore, deriving a \textit{data-efficient} and \textit{compute-efficient} sub-1-bit PTQ method remains an open challenge.

To bridge the gap between binary PTQ and QAT methods, we propose \name, an efficient and accurate PTQ method that can compress LLMs to \textit{sub-1-bit} levels.
By directly addressing multiple shortcomings of conventional post-training quantization methods, \name precisely initializes latent binary matrices and scales via robust Hessian-aware alternating direction method of multipliers (ADMM).
Then, \name utilizes a hierarchical reconstruction pipeline that optimizes parameters at the block level, and subsequently calibrates scaling factors at the model level for enhanced global activation alignment.
With only 128 calibration samples (0.26M tokens) and 1 GPU, \name achieves sub-binary compression and shows competitive performance in low-memory regimes.
\name enables compressing a 70B LLM from 137.95 GB to 5.75 GB using 1 GPU, and running the quantized 70B LLM on a consumer 8 GB GPU at up to 20.11 tokens per second, making LLM compression and inference accessible in resource-constrained environments.

\paragraph{Contributions.} Our main contributions are as follows:
\begin{itemize}[itemsep=2pt, topsep=0pt, parsep=5pt, partopsep=0pt, leftmargin=*]
    \item We propose \textsc{NanoQuant}, a post-training quantization (PTQ) method to compress LLMs to both 1-bit and sub-1-bit levels. This approach addresses the structural limitations of existing binary quantization frameworks.
    \item We provide a stability analysis of the initialization procedure and empirically show that precise low-rank binary initialization is critical for establishing a new sub-1-bit quantization frontier.
    \item We conduct extensive experiments across diverse model families and language tasks, demonstrating that \textsc{NanoQuant} achieves competitive performance with higher-bit PTQ and binary quantization-aware training (QAT) methods, despite using limited calibration data.
    \item We implement custom binary GEMV and GEMM CUDA kernels for \textsc{NanoQuant}, enabling significantly higher inference throughput, reduced memory footprints, and enhanced energy efficiency for datacenter GPUs, consumer GPUs, and edge devices.
\end{itemize}

\section{Related Work}
\label{sec:related_work}
\paragraph{Binary Post-Training Quantization.}
State-of-the-art binary post-training quantization (PTQ) methods often adopt in-place binarization and full-precision scales to preserve the sign and magnitude of weights, respectively \cite{bi_llm, arb_llm, hb_llm}.
Other methods introduce sparsity to further reduce the memory footprint of binary weights \cite{stb_llm}.
However, although such methods show respectable performance, these binary PTQ algorithms incur additional storage requirements (e.g. scaling factors and grouping bit-masks), causing them to fall short of their intended binary compression rates, requiring at least 2 or 3 bits per weight \cite{bi_llm, arb_llm, ptq1p61, hb_llm}.
QMoE \cite{qmoe} and BTC-LLM \cite{btc_llm} are notable methods that achieve sub-1-bit levels, but they respectively target mixture-of-experts models or utilize codebooks with additional storage overhead.

\paragraph{Binary Quantization-Aware Training.}
In contrast, binary quantization-aware training (QAT) methods successfully reach binary compression rates through end-to-end training on larger datasets.
Many previous binary QAT methods utilize in-place binarization to compress LLM weights to binary levels \cite{bitnet, onebit, binarymos, paretoq}.
More recent methods compress weights to low-rank binary matrices to reach binary and sub-1-bit compression levels~\cite{littlebit,dbf}.
However, although such low-rank binary methods display promising performance, they require copious amounts of data and compute, requiring multiple GPUs for multiple days to train on hundreds of millions or billions of tokens.
Such resource demands have also limited these methods to binarize only relatively smaller models, such as Llama-2-7B.

\paragraph{Alternating Direction Method of Multipliers.}
Alternating Direction Method of Multipliers (ADMM) is a classical framework for constrained optimization that alternates between augmented-Lagrangian subproblems and dual-variable updates~\cite{admm_glowinski1975approximation,admm_gabay1976dual,boyd2011distributed}. 
Although ADMM is best understood in convex optimization~\cite{boyd2004convex,bertsekas2014constrained}, it has also been studied in nonconvex and nonsmooth settings~\cite{hong2016convergence,wang2019global,yang2022proximal,JMLR:v25:21-0831}. 
This makes ADMM a natural fit for model compression problems that combine continuous reconstruction objectives with hard structural constraints. 
Prior work has applied ADMM to low-bit quantization~\cite{leng2018extremely,rnn_admm,xu2021mixed,huang2021alternating,dbf}, quantization-aware factorization~\cite{cherniuk2024quantization}, sparsification~\cite{nxmtransformer,alps,elsa}, and sparse-plus-low-rank LLM compression~\cite{admm3}. 
In \name, we use ADMM in this spirit to decouple weight reconstruction from the discrete low-rank binary structure, enabling efficient optimization of sub-1-bit LLM weight decompositions.

\section{\name}
\label{sec:method}
This section introduces \name, a post-training quantization (PTQ) method capable of compressing LLM weights to sub-1-bit levels.
\name derives high-fidelity low-rank binary representations by integrating a precise initialization subroutine directly into a block-wise reconstruction loop, followed by lightweight global calibration.

\subsection{Quantization Scheme}
\label{sec:quant_scheme}

We formulate sub-1-bit weight compression as a low-rank binary factorization problem, similar to \cite{littlebit, dbf}.
Let $\mathbb{B} = \{-1, +1\}$ denote the set of binary values.
For each linear layer weight $\mathbf{W} \in \mathbb{R}^{d_{\text{out}} \times d_{\text{in}}}$ in the transformer, we approximate the dense matrix using two low-rank binary matrices $\mathbf{U}_{\pm 1} \in \mathbb{B}^{d_{\text{out}} \times r}$ and $\mathbf{V}_{\pm 1} \in \mathbb{B}^{d_{\text{in}} \times r}$, alongside two full-precision scaling vectors: an output channel scale $\mathbf{s}_1 \in \mathbb{R}^{d_{\text{out}}}$ and an input channel scale $\mathbf{s}_2 \in \mathbb{R}^{d_{\text{in}}}$.

The decomposition structure is defined as
\begin{equation}
    \label{eq:structure}
    \mathbf{W} \approx \widehat{\mathbf{W}} = \mathbf{s}_1 \odot (\mathbf{U}_{\pm 1} \mathbf{V}_{\pm 1}^\top) \odot \mathbf{s}_2^\top,
\end{equation}
where $\odot$ denotes element-wise multiplication with broadcasting.
\autoref{fig:quant_scheme} visualizes this scheme, illustrating how the dense weight matrix decomposes into continuous latent factors and scales before binarization and packing.
Direct optimization of binary parameters constitutes a non-convex, combinatorial problem that is NP-hard \cite{froese2023training}.
To address this within a strict PTQ budget, \name employs a sequential block reconstruction pipeline that incorporates precise initialization and latent optimization.

\begin{figure}[t]
    \centering
    \includegraphics[width=1.0\linewidth]{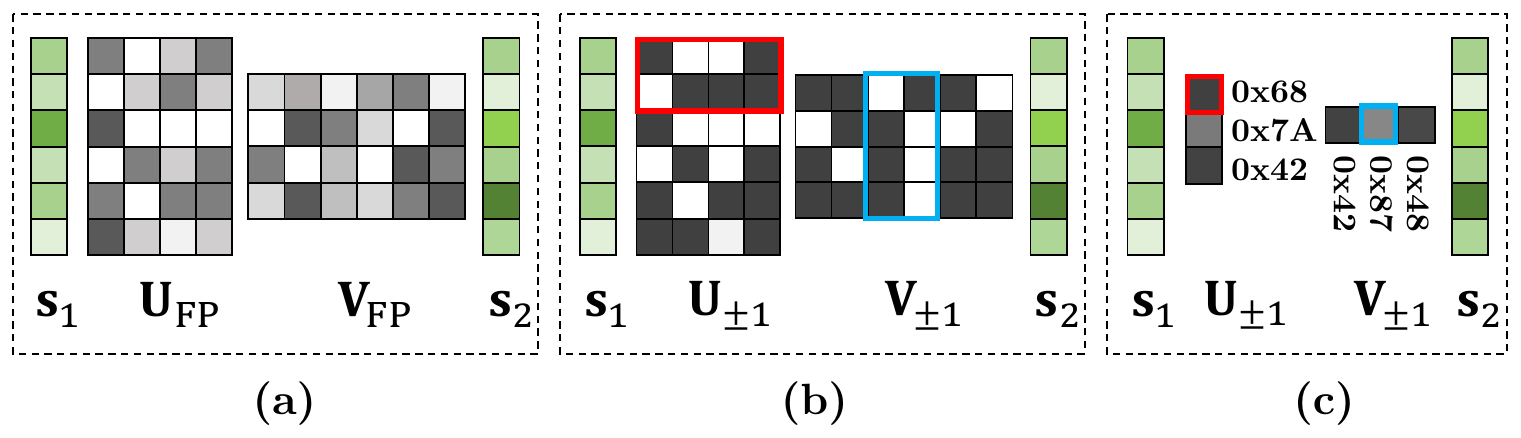}
    \vspace{-15px}
    \caption{Illustration of the \name compression scheme. The process proceeds in three stages: (a) \textbf{Factorization}, where the weight matrix is decomposed into continuous latent factors ($\mathbf{U}_\text{FP}, \mathbf{V}_\text{FP}$) and floating-point scales ($\mathbf{s}_1, \mathbf{s}_2$) which are fine-tuned to minimize reconstruction error; (b) \textbf{Binarization}, where these optimized factors are quantized into binary matrices ($\mathbf{U}_{\pm 1}, \mathbf{V}_{\pm 1}$) containing $\{-1, +1\}$ values; and (c) \textbf{Packing}, where these values are mapped to bits ($-1 \to 0, +1 \to 1$) and efficiently packed into integer formats (\eg 8-bit blocks) for memory efficiency.}
    \label{fig:quant_scheme}
\end{figure}

\subsection{Block Reconstruction Pipeline}
\label{sec:block_recon}

We sequentially compress each linear layer in each transformer block.
Unlike methods that treat initialization as a separate pre-processing phase, \name integrates initialization as a subroutine within the block reconstruction loop.
As depicted in~\autoref{fig:block_recon}, each block undergoes a three-step optimization process: (1) error propagation mitigation, (2) low-rank binary initialization via ADMM and magnitude balancing, and (3) factorized component refinement.

\paragraph{Step 1: Error Propagation Mitigation.}
Quantization error accumulates as the reconstruction proceeds through the network \cite{gptq}.
We tune the full-precision weights of the current block to minimize the error introduced by the quantization of preceding blocks, as well as previously factorized layers in the current block.
This comprehensive strategy is in line with recent quantization methods that adopt this method for either some \cite{dbf} or all ~\cite{quip_sharp, aqlm, qep} linear layers, and \name falls in the latter.
 
\paragraph{Step 2: Low-Rank Binary Initialization.}
Because PTQ relies on a small calibration set, the stability of initialization is critical \cite{hubara2021accurate,adaround}.
We initialize the low-rank binary parameters and scales through an activation-aware process involving preconditioning, factorization via alternating direction method of multipliers (ADMM), and magnitude balancing.

\paragraph{Step 2-1: Hessian-Aware Preconditioning.}
To minimize quantization error, we adopt the formulation from DBF \cite{dbf} and consider the second-order Taylor expansion of the task loss.
The objective minimizes the Hessian-weighted distortion approximated via Kronecker-factored approximate curvature (K-FAC)~\cite{kfac}:
\begin{equation}
\mathcal{L}(\widehat{\mathbf{W}})
\approx
\|\widetilde{\mathbf{D}}_{\text{out}} (\mathbf{W}-\widehat{\mathbf{W}})\widetilde{\mathbf{D}}_{\text{in}}\|_F^2.
\end{equation}
Here, $\widetilde{\mathbf{D}}_{\text{in}}$ and $\widetilde{\mathbf{D}}_{\text{out}}$ are diagonal preconditioners constructed from activation and gradient statistics.
These values are computed during a global calibration phase prior to the block-wise reconstruction loop, as outlined in \cref{alg:full_compact_fmt}.
Given limited calibration data, empirical estimates can be sensitive to outliers.
To mitigate this, we employ shrinkage \cite{ledoit2004well} regularization on the diagonal entries,
\begin{equation}
    [\widetilde{\mathbf{D}}_{(\cdot)}]_{ii} \leftarrow (1-\gamma)[\mathbf{D}_{(\cdot)}]_{ii} + \gamma\,\mathrm{mean}(\mathbf{D}_{(\cdot)}).
\end{equation}
The shrinkage coefficient $\gamma \in [0, 1]$ plays a pivotal role in regulating the trade-off between preserving feature-specific curvature information and maintaining global robustness against calibration noise.
We empirically find that smaller values (\eg 0.2) are optimal for Llama and Qwen models, and larger values (\eg 0.6) are optimal for Gemma 3 models and Rnj-1.

\paragraph{Step 2-2: Latent Binary Factorization (LB-ADMM).}
We formulate initialization as finding factors $\mathbf{U}$ and $\mathbf{V}$ that approximate the preconditioned target $\widetilde{\mathbf{W}}_{\text{target}}$.
To handle the non-convex landscape, we employ ADMM with ridge regularization $\lambda$, introducing auxiliary variables $\mathbf{Z}$ and scaled dual variables $\mathbf{\Lambda}$ to decouple constraints.
The optimization problem is defined as:
\begin{equation}
\label{eq:lb_admm_obj}
\begin{aligned}
\min_{\mathbf{U},\mathbf{V},\mathbf{Z_U},\mathbf{Z_V}}\;&\tfrac12\|\widetilde{\mathbf{W}}_{\text{target}}-\mathbf{U}\mathbf{V}^\top\|_F^2
+\tfrac{\lambda}{2}\big(\|\mathbf{U}\|_F^2+\|\mathbf{V}\|_F^2\big)\\
\text{s.t. }\;& \mathbf{U}=\mathbf{Z}_\mathbf{U},\;\mathbf{V}=\mathbf{Z}_\mathbf{V}.
\end{aligned}
\end{equation}

\begin{figure}
    \centering
    \includegraphics[width=1.0\linewidth]{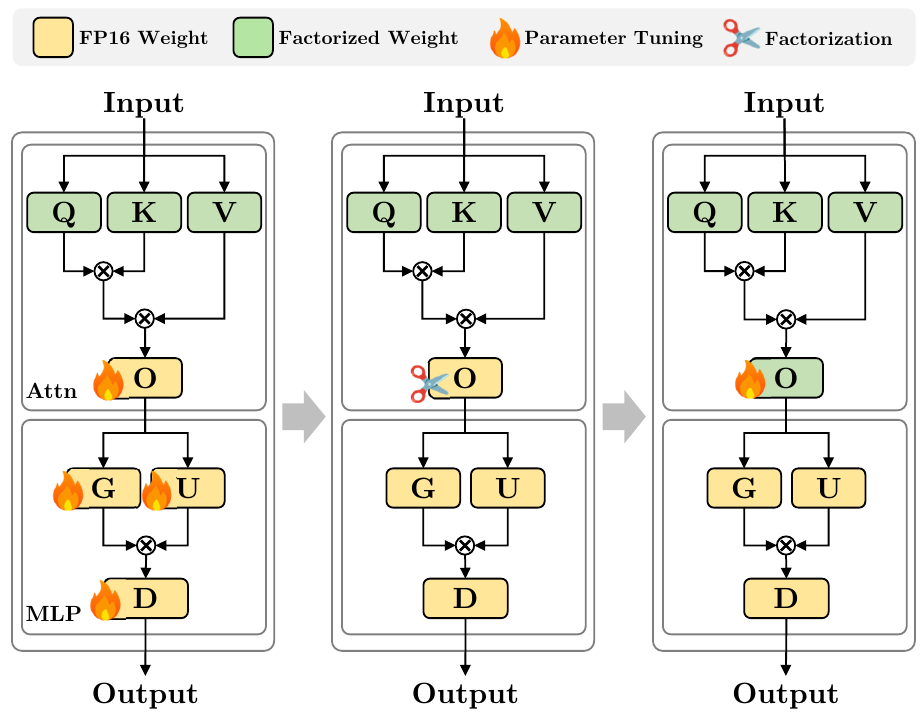}
    \caption{The \name block reconstruction pipeline for compressing linear layers. The process sequentially optimizes each transformer block through three key phases: (1) \textbf{Error Propagation Mitigation} to adjust full-precision weights for accumulated errors; (2) \textbf{Low-Rank Binary Initialization}, which utilizes Latent Binary ADMM (LB-ADMM) to precisely generate latent binary factors and scales; and (3) \textbf{Factorized Component Refinement}, which fine-tunes the continuous latent matrices and scales using Straight-Through Estimators (STE) before final packing.}
    \label{fig:block_recon}
\end{figure}

The solver alternates between updating continuous factors, auxiliary proxies, and dual variables.
First, we update $\mathbf{U}$ (symmetrically $\mathbf{V}$) by solving a linear system regularized by penalty $\rho$ and $\lambda$:
\begin{equation}
    \label{eq:admm_update}
    (\mathbf{V}^\top \mathbf{V} + (\rho + \lambda)\mathbf{I}) \mathbf{U}^\top = \mathbf{V}^\top \widetilde{\mathbf{W}}_{\text{target}}^\top + \rho(\mathbf{Z}_{\mathbf{U}} - \mathbf{\Lambda}_{\mathbf{U}})^\top.
\end{equation}
We employ stabilized Cholesky decomposition for this step, reducing the computational complexity to $\mathcal{O}(r^3/3)$ compared to general LU factorization, which scales as $\mathcal{O}(2r^3/3)$.
This optimization is pivotal, as it enables \name to scale efficiently to massive architectures (\eg Llama-2-70B) within limited computational budgets.

Second, we update the auxiliary variable $\mathbf{Z}$ using the consensus variable $\mathbf{P}_{\mathbf{U}} \triangleq \mathbf{U} + \mathbf{\Lambda}$.
We apply Sign-Value Independent Decomposition (SVID) \cite{ls_binary_quant,onebit} to derive the optimal rank-1 approximation that preserves the sign structure:
\begin{equation}
    \mathbf{Z}_{\mathbf{U}}^{(k+1)} = \mathrm{SVID}(\mathbf{P}_{\mathbf{U}}^{(k+1)}).
\end{equation}
Finally, we update the dual variables to enforce consensus, computed as $\mathbf{\Lambda}_{\mathbf{U}}^{(k+1)} = \mathbf{\Lambda}_{\mathbf{U}}^{(k)} + \mathbf{U}^{(k+1)} - \mathbf{Z}_{\mathbf{U}}^{(k+1)}$.

\paragraph{Step 2-3: Latent Magnitude Balancing.}
Upon convergence of ADMM, the pre-binary variables $\mathbf{P}_{\mathbf{U}}^{(K)}$ and $\mathbf{P}_{\mathbf{V}}^{(K)}$ possess inherent scale ambiguity, resulting in ill-conditioned proxies.
To rectify this, we first recover the unscaled continuous proxies, defined as $\widehat{\mathbf{U}} = \widetilde{\mathbf{D}}_{\text{out}}^{-1}\mathbf{P}_{\mathbf{U}}^{(K)}$ and $\widehat{\mathbf{V}} = \widetilde{\mathbf{D}}_{\text{in}}^{-1}\mathbf{P}_{\mathbf{V}}^{(K)}$, and compute an equilibrium factor $\eta$ to equalize their Frobenius norms:
\begin{equation}
    \eta = \sqrt{ \|\widehat{\mathbf{V}}\|_F \big/ \|\widehat{\mathbf{U}}\|_F }.
\end{equation}
The scaling vectors $\mathbf{s}_1$ and $\mathbf{s}_2$ are computed directly from the balanced projections of these proxies to capture the magnitude information via the mean absolute value:
\begin{equation}
    [\mathbf{s}_1]_i = \text{mean}(|\eta \boldsymbol{\widehat{u}}_i|), \quad [\mathbf{s}_2]_j = \text{mean}(|\eta^{-1} \boldsymbol{\widehat{v}}_j|),
\end{equation}
where $\boldsymbol{\widehat{u}}_i$ and $\boldsymbol{\widehat{v}}_j$ denote the row vectors of $\widehat{\mathbf{U}}$ and $\widehat{\mathbf{V}}$, respectively.
Following scale extraction, we define the final latent variables $\mathcal{U}$ and $\mathcal{V}$ to serve as well-conditioned initializers for the subsequent fine-tuning phase:
\begin{equation}
\begin{aligned}
    \mathcal{U} &\coloneqq \eta \widehat{\mathbf{U}} = \eta \widetilde{\mathbf{D}}_{\text{out}}^{-1} \mathbf{P}_{\mathbf{U}}^{(K)}, \\
    \mathcal{V} &\coloneqq \eta^{-1} \widehat{\mathbf{V}} = \eta^{-1} \widetilde{\mathbf{D}}_{\text{in}}^{-1} \mathbf{P}_{\mathbf{V}}^{(K)}.
\end{aligned}
\end{equation}
This separation allows the explicit scales to handle magnitude at the input and output boundaries, ensuring that the core linear transformation proceeds sequentially without intervening scalar operations, thereby reducing computational overhead on hardware accelerators.

\paragraph{Step 3: Factorized Component Refinement.}
Following initialization, we refine the factorized components of the current target linear layer to align with the full-precision block outputs.
Unlike approaches that defer binary optimization to a global stage \cite{dbf} through PV-tuning \cite{pv_tuning}, we locally optimize these parameters during the block reconstruction phase.
We jointly tune the continuous latent proxies $\mathcal{U}, \mathcal{V}$ and the scaling vectors $\mathbf{s}_1, \mathbf{s}_2$ using the Straight-Through Estimator (STE) \cite{ste}.
Let $\mathcal{B}(\cdot)$ and $\widehat{\mathcal{B}}(\cdot)$ denote the full-precision and quantized mappings of the current transformer block (with all previously processed blocks fixed), respectively.
The optimization objective is formulated as:
\begin{equation}
    \min_{\mathcal{U}, \mathcal{V}, \mathbf{s}_1, \mathbf{s}_2} \| \mathcal{B}(\mathbf{X}_{\text{in}}) - \widehat{\mathcal{B}}(\mathbf{X}_{\text{in}}; \mathrm{sign}(\mathcal{U}), \mathrm{sign}(\mathcal{V}), \mathbf{s}_1, \mathbf{s}_2) \|_F^2.
\end{equation}
This formulation allows gradients to propagate through the quantization function, enabling local identification of optimal sign structures while concurrently adjusting channel-wise magnitudes.
Upon convergence, we fix $\mathbf{U}_{\pm1} = \mathrm{sign}(\mathcal{U})$ and $\mathbf{V}_{\pm1} = \mathrm{sign}(\mathcal{V})$ as the final binary values, and pack the binary weights into integer values.

\subsection{Model Reconstruction}
With the block-wise optimization concluded, the binary parameters are frozen and packed into efficient integer formats.
Consequently, the final model reconstruction phase focuses exclusively on optimizing the floating-point scaling vectors $\mathbf{S}_{\text{global}} = \{\mathbf{s}_{1},\mathbf{s}_{2}\}_{\forall l}$ to align the predictive distributions of the quantized model with those of the full-precision model \cite{alphatuning}.
The objective function minimizes the Kullback-Leibler (KL) divergence:
\begin{equation}
    \label{eq:stage3_obj}
    \min_{\mathbf{S}_{\text{global}}} D_{\mathrm{KL}} \left( \text{softmax}(z_{\mathcal{M}}/T) \parallel \text{softmax}(z_{\mathcal{\hat{M}}}/T) \right),
\end{equation}
where 
$z_{\mathcal{M}} = \text{Logits}(M(X))$ and
$z_{\mathcal{\hat{M}}} = \text{Logits}(\hat{M}(X;S_{global}))$.
Unlike prior methods that require extensive memory resources for global fine-tuning \cite{efficientqat}, this approach maintains fixed bit-packed binary weights throughout the process.
This constraint substantially reduces the memory footprint, and it makes calibration of massive models, such as Llama-2-70B, feasible on a single GPU.

\begin{algorithm}[t!]
\caption{The \name algorithm.}
\label{alg:full_compact_fmt}
\footnotesize
\begin{algorithmic}[1]
\item[\textbf{Input:}] FP teacher $\mathcal{M}$, calibration set $\mathcal{X}_{\rm cal}$, rank $r$, robust-diagonal parameters $(\tau,\gamma)$, ADMM parameters $(K,\rho,\lambda,\epsilon)$, optimization steps $(T_{\rm pre},T_{\rm post},T_{\rm glob})$
\item[\textbf{Output:}] Quantized model $\widehat{\mathcal{M}}$ with packed binaries $\{\mathbf{U}_{\pm1}^{(\ell)},\mathbf{V}_{\pm1}^{(\ell)}\}$ and float scales $\{\mathbf{s}_1^{(\ell)},\mathbf{s}_2^{(\ell)}\}$
\hrule
\vspace{0.1cm}
\STATE \textcolor{gray!70!black}{\textit{\# Phase 1: Global Calibration}}
\FOR{each linear layer $\ell$}
    \STATE Run $\mathcal{X}_{\rm cal}$ through $\mathcal{M}$ and collect statistics $(\mathbf{z}_{\rm in}, \mathbf{z}_{\rm out})^{(\ell)}$
    \STATE $(\widetilde{\mathbf{D}}_{\rm in}, \widetilde{\mathbf{D}}_{\rm out})^{(\ell)} \leftarrow \textsc{RobustDiag}(\mathbf{z}_{\rm in}^{(\ell)}, \mathbf{z}_{\rm out}^{(\ell)};\tau,\gamma)$
\ENDFOR
\vspace{0.1cm}
\STATE \textcolor{gray!70!black}{\textit{\# Phase 2: Block Reconstruction Pipeline}}
\STATE $\widehat{\mathcal{M}} \leftarrow \mathcal{M}$ \hfill \textcolor{gray!70!black}{\scriptsize $\triangleright$ Init with teacher weights}
\FOR{block $b=1,\dots,B$}
    \STATE $\mathbf{X}_b \leftarrow \widehat{\mathcal{M}}_{<b}(\mathcal{X}_{\rm cal})$ \hfill \textcolor{gray!70!black}{\scriptsize $\triangleright$ input activation after already-compressed prefix}
    \STATE $\mathbf{Y}_b \leftarrow \mathcal{B}^{\text{FP}}_b(\mathbf{X}_b)$ \hfill \textcolor{gray!70!black}{\scriptsize $\triangleright$ teacher output for current block $b$}
    \STATE \textcolor{gray!70!black}{\textit{\footnotesize $\triangleright$ Step 1: Error Propagation Mitigation}}
    \STATE $\textsc{TuneFP}(b,\mathbf{X}_b,\mathbf{Y}_b;T_{\rm pre})$ 
    \STATE \textcolor{gray!70!black}{\textit{\footnotesize $\triangleright$ Step 2: Low-Rank Binary Initialization}}
    \FOR{linear layer $\ell \in b$ with weight $\mathbf{W}^{(\ell)}$}
        \STATE $\widetilde{\mathbf{W}}^{(\ell)} \leftarrow \widetilde{\mathbf{D}}_{\rm out}^{(\ell)}\,\mathbf{W}^{(\ell)}\,\widetilde{\mathbf{D}}_{\rm in}^{(\ell)}$
        \STATE $(\mathbf{U}, \mathbf{V}, \mathbf{s}_1, \mathbf{s}_2)^{(\ell)} \leftarrow \textsc{LB-ADMM}(\widetilde{\mathbf{W}}^{(\ell)}; r, K, \rho, \lambda, \epsilon)$
    \ENDFOR
    \STATE \textcolor{gray!70!black}{\textit{\footnotesize $\triangleright$ Step 3: Factorized Component Refinement}}
    \STATE $\textsc{TuneLatentSTE}(b,\mathbf{X}_b,\mathbf{Y}_b;T_{\rm post})$
    \FOR{linear layer $\ell \in b$}
        \STATE $\mathbf{U}_{\pm1}^{(\ell)} \leftarrow \mathrm{sign}(\mathbf{U}^{(\ell)})$; $\mathbf{V}_{\pm1}^{(\ell)} \leftarrow \mathrm{sign}(\mathbf{V}^{(\ell)})$
        \STATE $\textsc{PackBinary}(\mathbf{U}_{\pm1}^{(\ell)},\mathbf{V}_{\pm1}^{(\ell)})$
    \ENDFOR
\ENDFOR
\vspace{0.1cm}
\STATE \textcolor{gray!70!black}{\textit{\# Phase 3: Scale-Only Model Reconstruction}}
\STATE \textcolor{gray!70!black}{\textit{\scriptsize $\triangleright$ packed binaries are frozen during scale tuning}}
\STATE $\textsc{TuneScalesKD}(\widehat{\mathcal{M}},\mathcal{M},\mathcal{X}_{\rm cal};T_{\rm glob})$
\STATE \textbf{return} $\widehat{\mathcal{M}}$
\end{algorithmic}
\end{algorithm}

\section{Experiments}
\subsection{Experimental Setup}

\paragraph{Implementation and Environment.}
The implementation of \name relies on PyTorch \cite{pytorch} and the Transformers library \cite{transformers}.
Primary quantization and evaluation experiments used a single NVIDIA H100 (80 GB) GPU to ensure consistency across model scales up to the 70B parameter regime.
To assess deployment viability in resource-constrained environments, inference latency and memory footprints were analyzed on consumer-grade hardware, specifically an NVIDIA RTX 3050 (8 GB) GPU, and an edge device, an NVIDIA Jetson TX2.

\begin{table*}[!t]
    \small
    \centering
    \caption{
    WikiText-2 perplexity ($\downarrow$) results of 1-bit and sub-1-bit post-training quantization methods.
    The evaluation encompasses pre-trained models from the Llama-2 (L2), Llama-3 (L3), Gemma 3 (G3), Qwen3 (Q3), and Rnj-1 (R1) families.
    In these abbreviations, the numerical suffix denotes the parameter count in billions (\eg L3-8 represents Llama-3-8B).
    \name demonstrates performance competitive with higher-bit baselines across these architectures.
    }
    \label{tab:ppl_main}
    \setlength{\tabcolsep}{3.5pt}
    \resizebox{\textwidth}{!}{
    \begin{tabular}{lll | ccc | cccc | cccc | ccccc | c}
        \toprule
        \textbf{W Bits} & \textbf{Total Bits} & \textbf{Method} & L2-7 & L2-13 & L2-70 & L3-1 & L3-3 & L3-8 & L3-70 & G3-1 & G3-4 & G3-12 & G3-27 & Q3-0.6 & Q3-1.7 & Q3-4 & Q3-8 & Q3-14 & R1-8 \\
        \midrule
        16.00 & - & - & 5.47 & 4.88 & 3.32 & 9.74 & 7.81 & 6.24 & 2.86 & 10.60 & 7.39 & 5.86 & 4.88 & 12.66 & 9.39 & 7.89 & 7.00 & 6.37 & 8.19 \\
        \midrule
        \midrule

        & 1.00 & RTN     & 1.63e5 & 4.82e4 & 1.57e5 & 5.39e8 & 1.82e13 & 4.41e5 & 3.98e5 & 3.64e22 & 2.96e17 & 1.90e24 & 6.29e21 & 2.58e7 & 5.14e8 & 1.45e6 & 5.12e6 & 7.05e9 & 7.26e5 \\
        
        & 1.00 & XNOR    & 6.59e4 & 9.80e3 & 1.37e4 & 1.15e5 & 1.78e6 & 8.50e5 & 8.61e5 & 1.91e8 & 4.50e6 & 5.00e6 & 3.32e6 & 3.27e7 & 1.60e6 & 1.56e10 & 1.37e8 & 1.63e8 & 6.25e4 \\
        
        & 2.88 & BiLLM   & 19.87 & 13.29 & 8.75 & 323.16 & 55.43 & 31.20 & 93.36 & 144.72 & 37.08 & 262.83 & 31.21 & 3.17e3 & 858.09 & 78.36 & 29.62 & 13.50 & 20.71 \\
        
        & 4.13 & STBLLM  & 10.12 & 8.08 & 5.26 & 187.40 & 25.46 & 16.68 & 155.43 & 80.54 & 21.97 & 63.45 & 16.46 & 329.76 & 1.32e3 & 35.03 & 20.04 & 10.72 & 14.78 \\
        
        & 2.51 & ARB-LLM$_{\text{RC}}$ & 11.80 & 8.43 & 5.20 & 66.36 & 23.87 & 19.06 & \textbf{7.89} & 67.43 & 23.37 & 32.47 & 16.80 & 129.52 & 49.51 & 18.04 & 12.74 & 10.25 & 15.13 \\
        
        & 3.25 & HBLLM$_{\text{R}}$ & \textbf{7.60} & \textbf{6.27} & \textbf{4.56} & 36.00 & \textbf{15.99} & \textbf{11.82} & 8.88 & \textbf{28.58} & \textbf{12.92} & \textbf{19.22} & \textbf{9.08} & 78.58 & 35.14 & 14.73 & \textbf{10.51} & \textbf{8.37} & \textbf{11.90} \\
        
        \rowcolor{babyblue!40}
        \multicolumn{1}{l}{\cellcolor{white}\multirow{-6}{*}{1.00}} & \textbf{1.00} & \name & 
        10.34 & 8.71 & 6.52 & \textbf{25.59} & 17.90 & 14.97 & 11.32 & 35.30 & 19.64 & 24.70 & 27.21 & \textbf{27.56} & \textbf{19.21} & \textbf{14.29} & 12.47 & 10.92 & 15.45 \\
        \midrule

        & 4.00 & STBLLM (6:8) 
        & \textbf{11.24} & \textbf{8.97} & \textbf{5.94} & 314.19 & 39.00 & 20.19 & 78.21 & 123.60 & 26.01 & 95.83 & \textbf{26.14} & 3.63e3 & 1.96e3 & 44.99 & 24.19 & \textbf{12.50} & \textbf{18.07} \\
        
        \rowcolor{babyblue!40}
        \multicolumn{1}{l}{\cellcolor{white}\multirow{-2}{*}{0.80}} & \textbf{0.80} & \name & 
        12.20 & 10.14 & 7.61 & \textbf{33.08} & \textbf{22.09} & \textbf{18.16} & \textbf{13.75} & \textbf{50.15} & \textbf{25.38} & \textbf{32.84} & 28.83 & \textbf{33.79} & \textbf{25.31} & \textbf{19.33} & \textbf{14.83} & 12.88 & 20.06 \\
        \midrule

        & 3.50 & STBLLM (4:8)
        & 20.27 & 15.22 & \textbf{9.27} & 6.69e3 & 381.77 & 88.84 & 1.83e3 & 592.31 & 69.63 & 489.64 & 83.29 & 5.74e5 & 5.18e4 & 1.30e3 & 109.40 & 28.50 & 67.60 \\
        
        \rowcolor{babyblue!40}
        \multicolumn{1}{l}{\cellcolor{white}\multirow{-2}{*}{0.55}} & \textbf{0.55} & \name & 
        \textbf{16.66} & \textbf{13.46} & 9.82 & \textbf{49.01} & \textbf{32.33} & \textbf{25.69} & \textbf{19.69} & \textbf{78.22} & \textbf{40.69} & \textbf{45.29} & \textbf{32.98} & \textbf{52.94} & \textbf{33.74} & \textbf{32.86} & \textbf{20.04} & \textbf{17.06} & \textbf{32.62} \\
        \bottomrule
    \end{tabular}}
\end{table*}

\paragraph{Models and Datasets.}
Evaluations included a diverse set of LLM families, including Llama-2 \cite{llama2}, Llama-3 \cite{llama3}, Gemma 3 \cite{gemma3}, Qwen3 \cite{qwen3}, and Rnj-1 \cite{rnj1}, with sizes ranging from 0.6B to 70B parameters.
This range addresses the sensitivity of smaller models to quantization noise \cite{li2020train, gong2024makes} and challenges the compression latency limits of larger architectures.
Calibration used 128 samples from the WikiText-2 dataset \cite{wikitext} with a sequence length of 2048.
Evaluation metrics included perplexity for next-token prediction and zero-shot accuracy across six commonsense reasoning tasks: WinoGrande \cite{winogrande}, HellaSwag \cite{hellaswag}, BoolQ \cite{boolq}, ARC-Easy, ARC-Challenge \cite{arc}, and PIQA \cite{piqa}.
To evaluate the robustness of \name across different data distributions, we also conduct experiments using the C4 dataset \cite{c4}, as detailed in \autoref{appendix:additional_exp}.

\begin{table}[t]
    \centering
    \caption{
    Zero-shot accuracy comparison on commonsense reasoning tasks using Llama-3 (L3) and Qwen3 (Q3) models.
    \name maintains competitive accuracy against higher-bit binary PTQ baselines, despite utilizing a 1-bit representation.
    }
    \label{tab:main_zeroshot}
    \setlength{\tabcolsep}{2.0mm}
    \resizebox{1\linewidth}{!}
    {
    \begin{tabular}{lrl cccccccc}
        \toprule
        \textbf{Model} & 
        Bits & 
        Method &
        ARC-e & ARC-c & BoolQ & Hella. & Wino. & PIQA & Avg.
        \\
        \midrule

        & 16.00 & BF16
        & 81.57 & 51.45 & 81.96 & 60.01 & 73.56 & 80.09 & 71.44 \\

        \cmidrule{2-10}      
        & 4.13 & STBLLM & 36.87 & 19.97 & 48.01 & 36.47 & 57.62 & 61.48 & 39.83 \\       
        & 3.25 & HBLLM$_{\text{col}}$ & 60.02 & 29.10 & 63.03 & 43.46 & 63.77 & 70.35 & 50.45 \\
        & 2.88 & BiLLM & 36.32 & 18.34 & 56.36 & 30.16 & 51.93 & 57.56 & 38.16 \\
        & 2.51 & ARB-LLM$_{\text{RC}}$ & 49.71 & 22.95 & 64.28 & 34.73 & 56.04 & 63.28 & 44.23 \\
        & 2.28 & GPTQ(w2g64) & 28.24 & 20.14 & 41.87 & 27.00 & 50.59 & 54.08 & 36.99 \\

        \cmidrule{2-10}
        \rowcolor{babyblue!40}
        \cellcolor{white}\multirow{-8}{*}{L3-8} & \textbf{1.00} & \textbf{\name}
        & 43.69 & 20.31 & 61.47 & 33.81 & 55.96 & 60.45 & 45.95 \\
        \midrule

        & 16.00 & BF16
        & 81.52 & 52.47 & 83.03 & 58.81 & 72.38 & 79.11 & 71.22 \\

        \cmidrule{2-10}      
        & 4.13 & STBLLM & 52.78 & 27.13 & 62.84 & 38.38 & 57.54 & 64.15 & 46.15 \\        
        & 3.25 & HBLLM$_{\text{col}}$ & 68.43 & 36.77 & 68.53 & 45.68 & 63.85 & 71.65 & 54.30 \\      
        & 2.88 & BiLLM & 28.96 & 22.35 & 62.23 & 32.42 & 51.30 & 55.01 & 38.18 \\
        & 2.51 & ARB-LLM$_{\text{RC}}$ & 68.18 & 34.64 & 66.30 & 40.83 & 59.12 & 68.17 & 51.86 \\
        & 2.28 & GPTQ(w2g64) & 28.62 & 20.99 & 43.64 & 29.52 & 50.04 & 54.68 & 37.92 \\

        \cmidrule{2-10}
        
        \rowcolor{babyblue!40}
        \cellcolor{white}\multirow{-8}{*}{Q3-8} & \textbf{1.00} & \textbf{\name} & 49.45 & 24.32 & 62.17 & 36.34 & 58.01 & 63.32 & 48.94 \\

        \bottomrule 
    \end{tabular}}
\end{table}

\paragraph{Baselines.}
We benchmark \name against state-of-the-art binary post-training quantization (PTQ) methods, specifically BiLLM \cite{bi_llm}, ARB-LLM \cite{arb_llm}, STBLLM \cite{stb_llm}, and HBLLM \cite{hb_llm}.
Comparisons also include binary quantization-aware training (QAT) methods such as OneBit \cite{onebit}, BinaryMoS \cite{binarymos}, LittleBit \cite{littlebit}, and DBF \cite{dbf}.
We utilize official open-source implementations for PTQ baselines and select the highest-performing variants, such as ARB-LLM$_\text{RC}$ and HBLLM$_\text{col}$.
Regarding QAT methods, we report the performance metrics for OneBit and BinaryMoS directly from their original literature.
Conversely, we reproduce specific components of DBF and LittleBit to validate initialization strategies.

\begin{table}[!t]
    \centering
    \setlength{\tabcolsep}{4pt}
    \caption{
        Comparing the compression and resource efficiency of various quantization methods, when compressing Llama-2-7B on NVIDIA H100 GPUs.
        \name requires orders of magnitude less data and over an order of magnitude less GPU time to achieve binary quantization.
    }
    \label{tab:compare_model_size}
    \begin{threeparttable}
    \resizebox{1.0\linewidth}{!}{
    \begin{tabular}{@{}lcccccc@{}}
        \toprule
        \textbf{Method} & \textbf{PTQ/QAT} & \textbf{Bits} & \textbf{Model Size} & \textbf{Data} & \textbf{GPU Hours} & \textbf{PPL ($\downarrow$)} \\

        \midrule
        Full-Precision & & & 13.48 GB & & & 5.47 \\
        \hline
        GPTQ (W2g64) & PTQ & 2.28 & 2.37 GB & 0.26M & 0.1 & 21.00 \\
        \hline
        STBLLM    & PTQ & 4.13 & 4.07 GB & 0.26M & 0.9 & 10.12 \\
        HBLLM$_{\text{R}}$    & PTQ & 3.25 & 3.16 GB & 0.26M & 2.2 & 7.60 \\
        BiLLM     & PTQ & 2.88 & 2.85 GB & 0.26M & 1.1 & 19.87 \\
        ARB-LLM$_{\text{RC}}$   & PTQ & 2.51 & 2.55 GB & 0.26M & 1.3 & 11.80 \\
        \hline
        OneBit    & QAT & 1.04 & 1.37 GB & 155.46M & 700.7 & 9.73 \\
        BinaryMoS & QAT & 1.08 & 1.40 GB & 196.00M & 92.8 & 7.88 \\
        DBF       & QAT & 1.00 & 1.33 GB & 1.38B & 37.6 & 9.25 \\
        LittleBit & QAT & 1.04 & 1.33 GB & 196.00M & 123.6 & 9.08 \\
        \hline
        \rowcolor{babyblue!40} \name & PTQ & 1.00 & 1.33 GB & 0.26M & 1.7 & 10.34 \\
        \rowcolor{babyblue!40} \name & PTQ & 1.00 & 1.33 GB & 2.10M & 2.5 & 8.85 \\
        \bottomrule
    \end{tabular}
    }
    \end{threeparttable}
\end{table}

\subsection{Accuracy Analysis}

\paragraph{Next Token Prediction.}
\autoref{tab:ppl_main} presents the perplexity comparison between \name and existing binary PTQ baselines.
The results indicate that \name maintains functional perplexity across diverse model families while using fewer bits than competing methods.
Prior binary PTQ approaches often struggle to break the 1-bit barrier due to structural overhead, but \name achieves sub-1-bit compression without a catastrophic degradation in the predictive distribution.
This finding suggests that the proposed low-rank factorization effectively captures the salient weight information required for language modeling, even at extreme compression rates.

\vspace{-10pt}

\paragraph{Zero-Shot Reasoning.}
The evaluation of commonsense reasoning tasks in \autoref{tab:main_zeroshot} reveals that \name yields performance competitive with higher-bit binary PTQ baselines.
Furthermore, the method approaches the zero-shot accuracy of binary QAT methods.
This result is notable because QAT approaches typically require extensive training on billions of tokens.
In contrast, \name achieves comparable fidelity using orders of magnitude less data and compute.
This efficiency suggests that precise initialization and block-wise reconstruction may substitute for the expensive end-to-end retraining traditionally required for binary quantization.

\subsection{Compression vs. Model Size}

We analyze the storage efficiency of various quantization paradigms using Llama-2-7B as a case study, as shown in \autoref{tab:compare_model_size}.
The analysis reveals that standard binary PTQ methods often incur substantial overhead due to auxiliary parameters.
Consequently, their effective storage requirements exceed those of 2-bit quantization methods such as GPTQ \cite{gptq}.
For instance, BiLLM and STBLLM require 2.88 and 4.13 bits per weight, respectively, whereas GPTQ (W2g64) uses only 2.28 bits.
\name overcomes this limitation and achieves genuine 1-bit and sub-1-bit post-training quantization.
By minimizing metadata overhead, it offers a storage solution that is strictly more efficient than existing PTQ baselines while maintaining perplexity levels competitive with resource-intensive QAT methods.
A comprehensive breakdown of the storage requirements and the mathematical derivation of effective bits per weight (BPW) for both \name and existing baselines can be found in \autoref{appendix:model_size}.

\begin{figure}[t]
    \centering
    \includegraphics[width=1.0\linewidth]{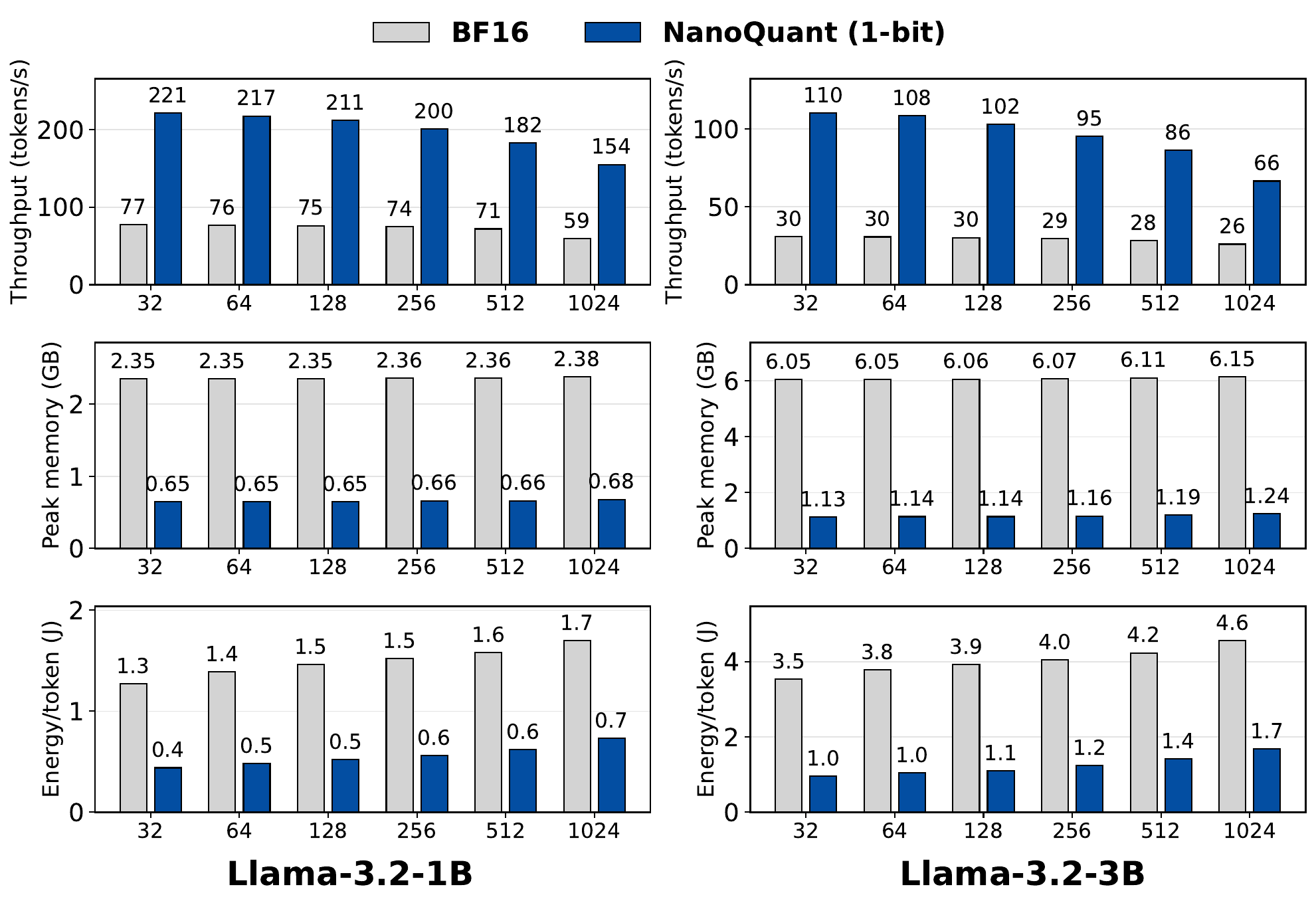}
    \caption{
    On an NVIDIA RTX 3050 (8 GB) GPU, \name delivers up to $3.6\times$ higher decoding throughput, $5.4\times$ lower peak memory usage, and $3.9\times$ greater energy efficiency compared to \texttt{BF16} baselines for Llama-3.2-1B and 3B models.
    }
    \label{fig:rtx_3050_inference}
\end{figure}

\subsection{Inference Efficiency}
\label{subsec:inference}

The extreme compression ratio achieved by \name translates directly into reduced memory footprints and enhanced throughput, particularly in memory-bound regimes.
To quantify this, we compared the decoding performance of \name against a PyTorch \texttt{BF16} baseline.
We focused on this comparison because optimized inference kernels for the binary PTQ baselines in \autoref{tab:ppl_main} are currently unavailable.

\paragraph{Consumer Hardware.}
On an NVIDIA RTX 3050 (8 GB) GPU, \name enables up to a $3.6\times$ speedup in inference throughput for Llama-3.2-3B.
Additionally, the method achieves a $5.4\times$ reduction in peak memory usage and a $3.9\times$ improvement in energy efficiency per token, as shown in \autoref{fig:rtx_3050_inference}.
Beyond speed, the method fundamentally expands accessibility.
\name compresses the Llama-2-70B model from 137.95 GB to 5.75 GB, representing a $24\times$ compression factor.
This reduction allows a 70B parameter model to fit entirely within the VRAM of a consumer-grade 8 GB GPU and effectively lowers the barrier to entry for large-scale model deployment.
To evaluate deployment viability in even more constrained environments, we extended our analysis to embedded systems.
Detailed performance metrics on the NVIDIA Jetson TX2 are provided in \autoref{appendix:kernel_impl}.

\begin{figure}[t]
    \centering
     \includegraphics[width=1.0\linewidth]{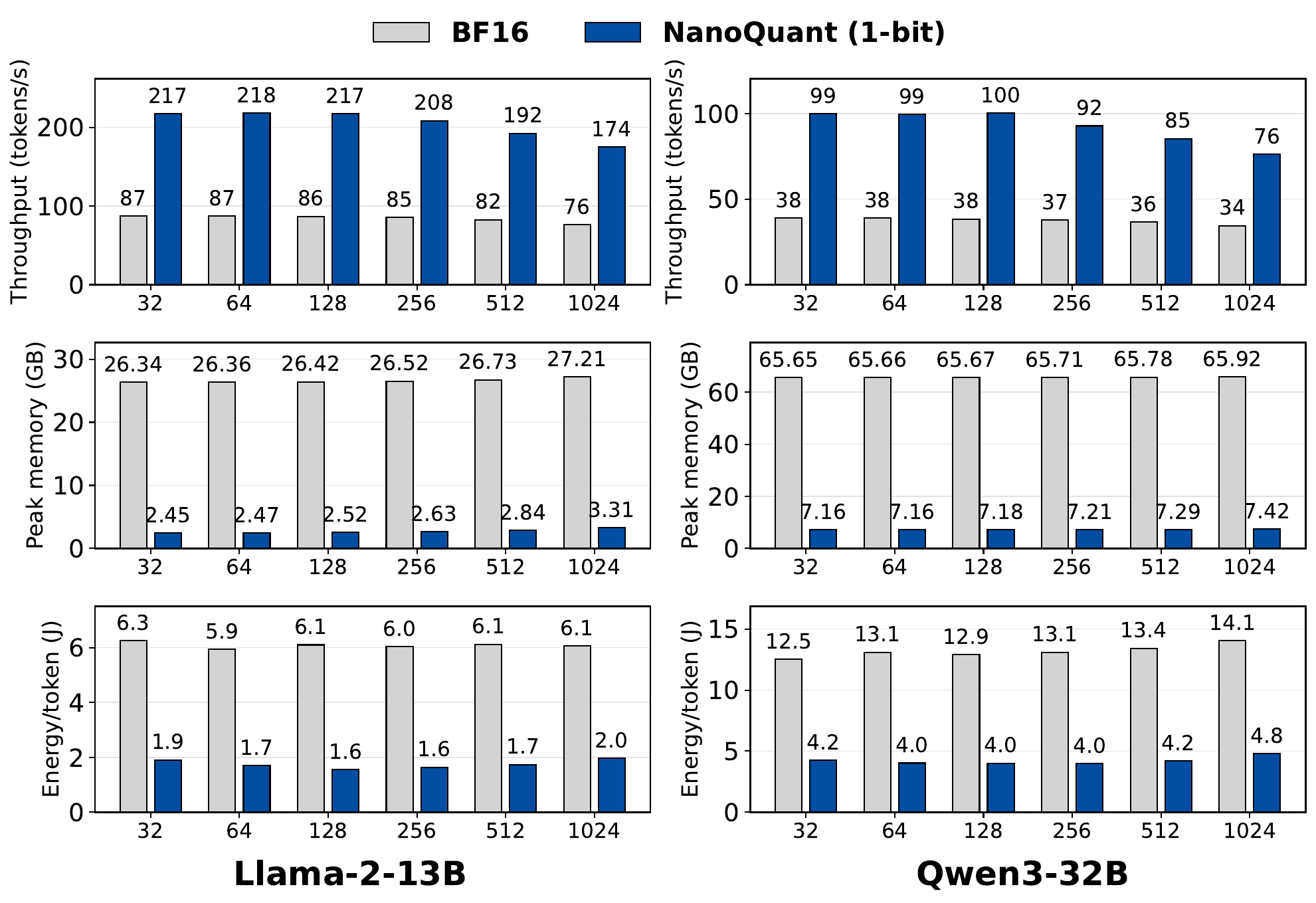}
    \caption{Datacenter inference efficiency on a single NVIDIA H100 (80 GB) GPU. \name enables faster decoding throughput while maintaining superior memory and energy efficiency for Llama-2-13B and Qwen3-32B, compared to the PyTorch \texttt{BF16} baseline.}
    \label{fig:h100_inference}
\end{figure}

\begin{figure*}[t!]
    \centering
    \includegraphics[width=0.9\linewidth]{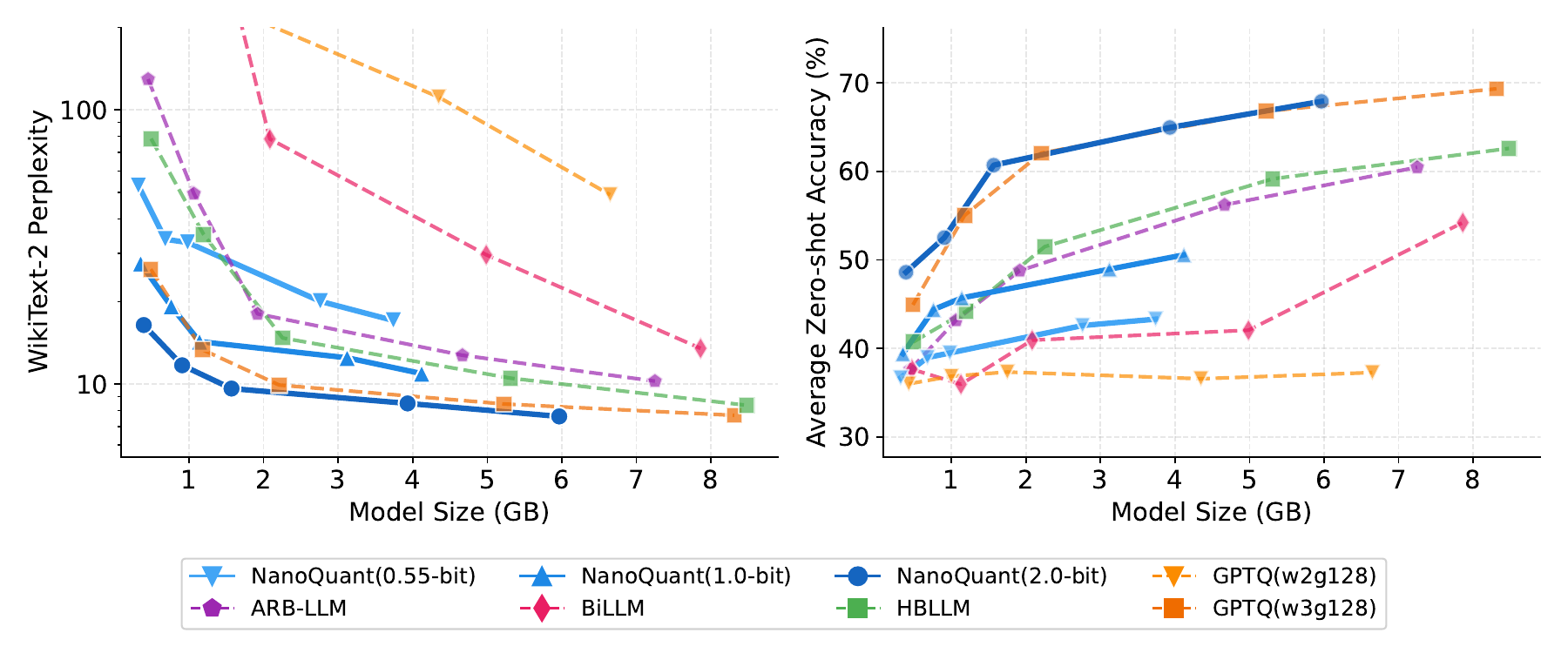}
    \caption{
    Pareto optimality analysis for models in the Qwen3 family (0.6B, 1.7B, 4B, 8B, 14B).
    \name establishes a new efficiency frontier in the low-bit regime, offering superior accuracy-per-bit trade-offs compared to existing state-of-the-art binary PTQ methods.
    }
    \label{fig:qwen3_pareto}
\end{figure*}

\paragraph{Datacenter Hardware.}
On high-end hardware (NVIDIA H100 (80 GB) GPU), \name alleviates memory bandwidth bottlenecks and demonstrates up to $10\times$ lower memory usage during inference.
As illustrated in \autoref{fig:h100_inference}, this results in faster single-batch inference and superior energy efficiency compared to the \texttt{BF16} baseline.
Additional kernel implementation details are provided in \autoref{appendix:kernel_impl}.

\subsection{Ablation Studies}

\begin{table}[t!]
    \centering
    \caption{Initialization via latent-binary ADMM (LB-ADMM) from \name outperforms other low-rank binary initialization strategies, when compressing Rnj-1 \cite{rnj1} to 0.8 bits.}
    \label{tab:compare_init}
    \resizebox{0.9\columnwidth}{!}{%
        \tabcolsep=0.2cm
        \renewcommand{\arraystretch}{1.1}
        \begin{tabular}{lcc}
            \toprule
            \textbf{Initialization Method} & \textbf{PPL} ($\downarrow$) & \textbf{Zero-shot} ($\uparrow$) \\
            \midrule
            Dual-SVID \cite{littlebit} & 167.73 & 35.11 \\
    
            DBF ADMM \cite{dbf} & 30.27 & 37.20 \\
    
            LB-ADMM (Ours) & \textbf{20.06} & \textbf{39.29} \\
            \bottomrule
        \end{tabular}
    }
\end{table}

\begin{table}[t!]
    \centering
    \caption{
    Component-wise efficacy analysis of the \name pipeline on Qwen3-8B-Base.
    The table demonstrates the contribution of each module---Initialization, Error Mitigation, Factorized Component Refinement, and Model Reconstruction---towards enhancing performance when combined.
    }
    \label{tab:compare_method_parts}
    \vspace{-5px}
    \resizebox{0.95\columnwidth}{!}{%
        \tabcolsep=0.1cm
        \setlength{\tabcolsep}{4pt}
        \renewcommand{\arraystretch}{1.00}
        \begin{tabular}{@{}ccc @{\hskip 0.3cm} c cc@{}}
            \toprule

            \multicolumn{3}{c}{\textbf{Block Reconstruction}} & 
            \multirow{2}{*}[-0.7em]{\makecell{\textbf{Model}\\\textbf{Recon.}}} & 
            \multirow{2}{*}[-0.7em]{\makecell{\textbf{PPL} ($\downarrow$)}} & 
            \multirow{2}{*}[-0.7em]{\makecell{\textbf{Zero-Shot} ($\uparrow$)}} \\
            
            \cmidrule(lr){1-3}
            
            \makecell{\textbf{Initialization}} & 
            \makecell{\textbf{Error}\\\textbf{Mitigation}} & 
            \makecell{\textbf{Fact.}\\\textbf{Refinement}} & 
             & & \\
            
            \midrule
            \cmark & \xmark & \xmark & \xmark & 206.03 & 36.89\\

            \cmark & \cmark & \xmark & \xmark & 15.07 & 46.40 \\

            \cmark & \xmark & \cmark & \xmark & 15.00 & 47.88 \\

            \cmark & \cmark & \cmark & \xmark & 13.58 & 46.75 \\

            \cmark & \cmark & \cmark & \cmark & \textbf{12.47} & \textbf{48.94} \\
            
            \bottomrule
        \end{tabular}%
    }
    \vspace{-10px}
\end{table}

\paragraph{Initialization Strategy.}
We investigate the hypothesis that precise initialization of low-rank binary matrices is critical for convergence in PTQ.
We integrated initialization strategies from prominent QAT methods, specifically LittleBit \cite{littlebit} and DBF \cite{dbf}, into our reconstruction pipeline.
\autoref{tab:compare_init} demonstrates that Latent-Binary ADMM (LB-ADMM) outperforms these alternatives in both perplexity and zero-shot tasks.
This result indicates that solving the combinatorial problem of binary factorization prior to fine-tuning provides a more stable optimization landscape than the initialization schemes used in existing QAT frameworks.

\paragraph{Component Efficacy.}
\autoref{tab:compare_method_parts} dissects the contribution of each algorithmic component.
This improvement stems from the combination of robust initialization and block-wise reconstruction.
Each module contributes distinctly to preserving the model's representational capacity under extreme compression.

\paragraph{Pareto Optimality.}
Finally, we analyze the trade-off between model size and performance across the Qwen3 family, as shown in \autoref{fig:qwen3_pareto}.
\name establishes a new Pareto frontier in the low-memory regime and consistently outperforms previous binary PTQ baselines.
These results suggest that low-rank binary representations are a viable alternative to low-bit integer quantization for memory-critical applications.

\begin{table}[t!] 
    \centering
    \caption{
    \name achieves comparable performance with QAT methods DBF and LittleBit, while using orders of magnitude less data and compute time, when compressing Qwen3-4B-Base and Llama-2-7B to 1 bit.}
    \label{tab:ptq_vs_qat}
    
    \small
    \resizebox{1.0\columnwidth}{!}{%
    \begin{tabular}{llcccc}
        \toprule
        \textbf{Model} & \textbf{Method} & \textbf{Data} & \textbf{GPU Hours} & \textbf{PPL} ($\downarrow$) & \textbf{Zero-shot} ($\uparrow$) \\
        \midrule

                & LittleBit & 169.50M & 92.5 & 14.79 & 47.32 \\
        Q3-4B   & DBF       & 1.19B & 25.3 & 14.62 & \textbf{52.30} \\
                & \name     & 1.05M & \textbf{2.3} & \textbf{12.62} & 50.63 \\
        
        \midrule

                & LittleBit & 196.00M & 123.6 & 9.08 & \textbf{54.92} \\
        L2-7B   & DBF       & 1.38B & 37.6 & 9.25 & 54.24 \\
                & \name     & 1.05M & \textbf{2.1} & \textbf{9.01} & 51.01 \\

        \bottomrule
    \end{tabular}
    }
\end{table}

\begin{figure*}[t!]
    \centering
     \includegraphics[width=1.0\linewidth]{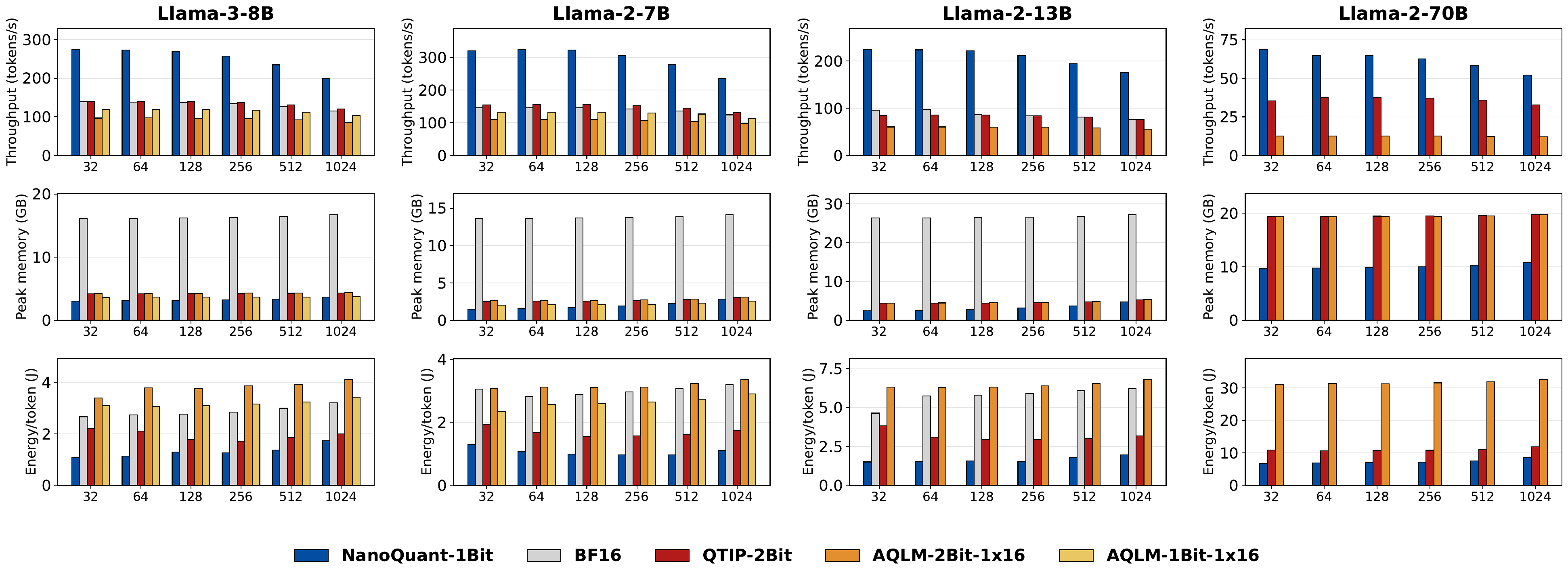}
    \caption{
    LLM decoding performance of \name, compared with PyTorch \texttt{BF16} and vector quantization methods (AQLM, QTIP) for 128 input tokens and various output sequence lengths, on an NVIDIA H100 (80 GB) GPU.
    \name shows superior inference speed, memory efficiency, and energy efficiency, compared to vector quantization methods and \texttt{BF16}.
    }
    \label{fig:h100_gemv_nano_vs_vq}
    \vspace{-5px}
\end{figure*}

\paragraph{Comparison with Low-Rank Binary QAT.}
\autoref{tab:ptq_vs_qat} contrasts \name with state-of-the-art QAT methods.
DBF \cite{dbf} and LittleBit \cite{littlebit} rely on training over 1 billion and 100 million tokens respectively.
We find that with 512 calibration samples, \name achieves comparable predictive performance with binary QAT methods.
This data efficiency validates the effectiveness of the proposed PTQ formulation for scenarios where full-scale retraining is impractical.

\paragraph{\name vs Vector Quantization.}
We compare the performance of \name with state-of-the-art 2-bit vector quantization methods \cite{aqlm}, PV-tuning \cite{pv_tuning}, and QTIP \cite{qtip}.
These methods utilize significantly more data and compute than other low-bit PTQ methods, as in \autoref{tab:nq_vs_vector_quant}.
Notably, compressing Llama-2-7B with \name takes less than 3 hours on an NVIDIA H100 (80 GB) GPU, while AQLM is reported to take at least 1 day on multiple A100 GPUs, and PV-Tuning even longer.
Nevertheless, \name shows competitive performance with data- and compute-intensive vector quantization baselines, making \name a practical compression method for resource-constrained environments.

\begin{table}[t]
    \centering
    \caption{
    Comparison of \name with state-of-the-art vector quantization methods
    AQLM~\cite{aqlm}, PV-Tuning~\cite{pv_tuning}, and QTIP~\cite{qtip}
    for compressing Llama-2-7B.
    Zero-shot denotes the average over ARC-C, ARC-E, HellaSwag, PIQA, and WinoGrande.
    \name achieves competitive performance under extreme compression while using substantially less calibration data.
    }
    \label{tab:nq_vs_vector_quant}

    \small
    \resizebox{\linewidth}{!}{
    \begin{tabular}{ccccccc}
        \toprule
        \textbf{Target} &
        \textbf{Method} &
        \textbf{Bits} &
        \textbf{Model Size} &
        \textbf{Data} &
        \textbf{PPL} ($\downarrow$) &
        \textbf{Zero-shot} ($\uparrow$) \\
        \midrule

        \multirow{5}{*}{2-bit}
            & QTIP
            & 2.02 & 2.15 GB & 12.58M
            & 6.29 & \textbf{62.05} \\
        \cmidrule{2-7}
            & AQLM
            & 2.02 & 2.15 GB & 8.00M
            & 6.64 & 56.47 \\
        \cmidrule{2-7}
            & AQLM + PV
            & 2.02 & 2.15 GB & 8.00M
            & \textbf{5.84} & 61.35 \\
        \cmidrule{2-7}
            & \name
            & \textbf{2.00} & \textbf{2.14 GB} & \textbf{1.05M}
            & 6.52 & 59.09 \\

        \midrule

        \multirow{2.5}{*}{1.5-bit}
            & AQLM + PV
            & 1.58 & 1.81 GB & 8.00M
            & 7.32 & 55.22 \\
        \cmidrule{2-7}
            & \name
            & \textbf{1.50} & \textbf{1.81 GB} & \textbf{1.05M}
            & \textbf{7.08} & \textbf{55.81} \\

        \midrule

        \multirow{2.5}{*}{1-bit}
            & AQLM + PV
            & 1.02 & 1.34 GB & 8.00M
            & \textbf{8.28} & \textbf{50.21} \\
        \cmidrule{2-7}
            & \name
            & \textbf{1.00} & \textbf{1.33 GB} & \textbf{1.05M}
            & 8.99 & 47.81 \\

        \bottomrule
    \end{tabular}
    }
    \vspace{-10px}
\end{table}

\subsection{Limitations and Future Work}
Although our experiments demonstrate data efficiency using a small calibration set, scaling the data and compute budget could enhance performance on more complex reasoning tasks.
Regarding inference, the custom CUDA kernels demonstrate promising results, as detailed in \autoref{subsec:inference}.
However, further optimization for next-generation architectures such as NVIDIA Blackwell GPUs or edge-specific hardware could yield additional speedups.
Additionally, while \name outperforms some 2-bit baselines, further enhancing capabilities to outperform higher-bit 2 or 3-bit PTQ performance remains an open challenge for the sub-binary regime.
Future work will focus on optimizing the compression runtime, exploring the scalability of the method to larger calibration datasets, and investigating adaptive rank allocation across layers to further optimize the accuracy-per-bit Pareto frontier.

\section{Conclusion}

We propose \name, an efficient and accurate post-training quantization method to enable 1-bit and sub-1-bit weight quantization of LLMs of up to 70B parameters, with only a single GPU.
\name enables rapid binarization and extremely compact weight storage, significantly reducing the memory footprint of LLM inference.
Custom binary CUDA kernels further improve energy efficiency and decoding speed.   
Our approach achieves up to $24\times$ model compression, making it feasible to run a 70B-parameter LLM on an 8 GB GPU. 
\name lowers the hardware barrier for some large-model inference by enabling fast, efficient compression and inference for researchers and developers in resource-constrained settings, and advances the frontier of extreme LLM quantization.

\clearpage

\section*{Acknowledgments}

We thank Drs. Heonjae Ha and Sangjeong Lee for their guidance, support and insightful feedback throughout this project.

\section*{Impact Statement}

This work presents \name, a sub-1-bit post-training quantization algorithm for large language models (LLMs).
By enabling the deployment of massive models (\eg Llama-2-70B) on consumer hardware (\eg a single 8 GB GPU) and edge devices, our method significantly lowers the barrier to entry for advanced AI research and application.
This contributes to the democratization of AI, allowing individuals and institutions with limited computational resources to utilize state-of-the-art LLMs.
Furthermore, \name is resource- and energy-efficient, as it drastically reduces the memory bandwidth and energy consumption required for inference, as demonstrated by our energy-efficiency experiments.

\bibliography{icml2026}
\bibliographystyle{icml2026}


\newpage
\appendix
\onecolumn

\crefalias{section}{appendix}
\crefalias{subsection}{appendix}
\crefname{appendix}{Appendix}{Appendices}

\section{Theoretical Analysis: Magnitude Balancing}
\label{appendix:theoretical_analysis}

In this section, we justify the magnitude balancing step used in \name.
The low-rank factorization $\mathbf{W} \approx \mathcal{U}\mathcal{V}^{\top}$ is scale-invariant: for any scalar $\eta > 0$,
\begin{equation}
    \mathcal{U}\mathcal{V}^{\top}
    =
    (\eta \mathcal{U})(\eta^{-1}\mathcal{V})^{\top}.
\end{equation}
Thus, the same reconstructed weight matrix can be represented by infinitely many pairs of latent factors with different relative magnitudes.
Without controlling this ambiguity, one factor can become excessively large while the other becomes excessively small, which can lead to poorly conditioned updates and unstable scale extraction.
Magnitude balancing selects a representative from this equivalent family whose two factors have comparable Frobenius norms before extracting the channel-wise scales.

In our formulation, the continuous latent variables are decomposed into channel-wise scales and binary directions:
\begin{equation}
    \mathcal{U} \approx \operatorname{diag}(\mathbf{s}_1)\mathbf{U}_{\pm 1},
    \qquad
    \mathcal{V} \approx \operatorname{diag}(\mathbf{s}_2)\mathbf{V}_{\pm 1},
\end{equation}
where $\mathbf{s}_1 \in \mathbb{R}^{m}$ and $\mathbf{s}_2 \in \mathbb{R}^{n}$ capture output-channel and input-channel magnitudes, respectively.

\subsection{Magnitude Balancing under Scale Ambiguity}

\begin{proposition}[Balanced representative under scale ambiguity]
\label{prop:magnitude_balancing}
Let $\hat{\mathcal{U}}$ and $\hat{\mathcal{V}}$ be nonzero latent factors obtained after LB-ADMM.
Among all equivalent rescalings
\begin{equation}
    \mathcal{U}_{\eta} = \eta \hat{\mathcal{U}},
    \qquad
    \mathcal{V}_{\eta} = \eta^{-1}\hat{\mathcal{V}},
    \qquad \eta > 0,
\end{equation}
which preserve $\mathcal{U}_{\eta}\mathcal{V}_{\eta}^{\top}
= \hat{\mathcal{U}}\hat{\mathcal{V}}^{\top}$, the value
\begin{equation}
    \eta^{\star}
    =
    \sqrt{
    \frac{\|\hat{\mathcal{V}}\|_F}
         {\|\hat{\mathcal{U}}\|_F}
    }
\end{equation}
minimizes the total factor magnitude
\begin{equation}
    \mathcal{J}(\eta)
    =
    \frac{1}{2}
    \left(
    \|\eta \hat{\mathcal{U}}\|_F^2
    +
    \|\eta^{-1}\hat{\mathcal{V}}\|_F^2
    \right).
\end{equation}
For this minimum-energy representative, the balanced factors satisfy
\begin{equation}
    \|\mathcal{U}_{\eta^\star}\|_F
    =
    \|\mathcal{V}_{\eta^\star}\|_F.
\end{equation}
\end{proposition}

\begin{proof}
Expanding the objective gives
\begin{equation}
    \mathcal{J}(\eta)
    =
    \frac{1}{2}
    \left(
    \eta^2 \|\hat{\mathcal{U}}\|_F^2
    +
    \eta^{-2}\|\hat{\mathcal{V}}\|_F^2
    \right).
\end{equation}
Differentiating with respect to $\eta$ yields
\begin{equation}
    \frac{d\mathcal{J}}{d\eta}
    =
    \eta\|\hat{\mathcal{U}}\|_F^2
    -
    \eta^{-3}\|\hat{\mathcal{V}}\|_F^2.
\end{equation}
Setting the derivative to zero gives
\begin{equation}
    \eta^4
    =
    \frac{\|\hat{\mathcal{V}}\|_F^2}
         {\|\hat{\mathcal{U}}\|_F^2},
\end{equation}
and therefore
\begin{equation}
    \eta^\star
    =
    \sqrt{
    \frac{\|\hat{\mathcal{V}}\|_F}
         {\|\hat{\mathcal{U}}\|_F}
    }.
\end{equation}
Substituting $\eta^\star$ into the two factor norms gives
\begin{equation}
    \|\eta^\star\hat{\mathcal{U}}\|_F^2
    =
    \|\hat{\mathcal{U}}\|_F\|\hat{\mathcal{V}}\|_F
    =
    \|(\eta^\star)^{-1}\hat{\mathcal{V}}\|_F^2,
\end{equation}
which proves that the selected rescaling equalizes the Frobenius norms of the two factors.
\end{proof}

This result directly motivates the magnitude balancing step in \name.
After LB-ADMM, we recover the unscaled latent factors and apply
\begin{equation}
    \mathcal{U} \leftarrow \eta^\star \hat{\mathcal{U}},
    \qquad
    \mathcal{V} \leftarrow (\eta^\star)^{-1}\hat{\mathcal{V}}.
\end{equation}
The channel-wise scales $\mathbf{s}_1$ and $\mathbf{s}_2$ are then extracted from these balanced factors.
This prevents one side of the factorization from absorbing most of the magnitude, which would otherwise make the subsequent scale extraction sensitive to numerical range and rounding effects.

\subsection{Conditioning Intuition}

Magnitude balancing also improves numerical stability.
The ADMM update for one factor involves solving a linear system with a matrix of the form
\begin{equation}
    \mathbf{H}
    =
    \mathbf{V}^{\top}\mathbf{V}
    +
    (\rho+\lambda)\mathbf{I},
\end{equation}
where $\rho$ is the ADMM penalty and $\lambda$ is the ridge regularization coefficient.
The condition number is
\begin{equation}
    \kappa(\mathbf{H})
    =
    \frac{\sigma_{\max}(\mathbf{V})^2+\rho+\lambda}
         {\sigma_{\min}(\mathbf{V})^2+\rho+\lambda}.
\end{equation}
The positive diagonal shift $(\rho+\lambda)\mathbf{I}$ ensures that the system is positive definite, but the scale ambiguity of the factorization can still create undesirable numerical regimes.
If $\|\mathbf{V}\|_F$ becomes extremely large, the conditioning of $\mathbf{H}$ increasingly reflects the conditioning of the Gram matrix $\mathbf{V}^{\top}\mathbf{V}$.
If $\|\mathbf{V}\|_F$ becomes extremely small, the data-dependent terms in the update become weak relative to the penalty and regularization terms.

By selecting the balanced representative of the equivalent factorization, \name avoids these extreme regimes before the subsequent refinement stage.
Thus, magnitude balancing does not change the reconstructed matrix $\mathcal{U}\mathcal{V}^{\top}$, but it provides a better-conditioned parameterization for scale extraction and straight-through estimator refinement.

\section{Stability Analysis of Low-Rank Binary Initialization}
\label{appendix:convergence_phase1}

This section analyzes the numerical stability of the robust Hessian preconditioning and Latent Binary ADMM (LB-ADMM) initialization used in \name. 
The LB-ADMM objective is nonconvex and includes discrete proxy updates; therefore, we do not claim global convergence or monotonic descent of the augmented Lagrangian. 
Instead, we show that the preconditioned target remains spectrally controlled and that the continuous ADMM subproblems are well-posed symmetric positive definite linear systems. 
These properties justify LB-ADMM as a stable initialization procedure for the subsequent block reconstruction stage.

\subsection{Problem Setup}

Let $\mathbf{W}_{\mathrm{FP}}\in\mathbb{R}^{m\times n}$ denote a full-precision weight matrix. 
We define the robustly preconditioned target as
\begin{equation}
    \mathbf{W}
    \triangleq
    \widetilde{\mathbf{D}}_{\mathrm{out}}
    \mathbf{W}_{\mathrm{FP}}
    \widetilde{\mathbf{D}}_{\mathrm{in}}
    \in \mathbb{R}^{m\times n},
\end{equation}
where $\widetilde{\mathbf{D}}_{\mathrm{out}}\in\mathbb{R}^{m\times m}$ and $\widetilde{\mathbf{D}}_{\mathrm{in}}\in\mathbb{R}^{n\times n}$ are diagonal preconditioners constructed from robust calibration statistics.

LB-ADMM seeks a rank-$R$ factorization $\mathbf{W}\approx \mathbf{U}\mathbf{V}^{\top}$ with $\mathbf{U}\in\mathbb{R}^{m\times R}$ and $\mathbf{V}\in\mathbb{R}^{n\times R}$. 
To separate continuous reconstruction from the discrete proxy structure, we introduce auxiliary variables $\mathbf{Z}_{\mathbf{U}}$ and $\mathbf{Z}_{\mathbf{V}}$. 
The regularized constrained objective is
\begin{equation}
\begin{aligned}
\min_{\mathbf{U},\mathbf{V},\mathbf{Z}_{\mathbf{U}},\mathbf{Z}_{\mathbf{V}}}
\quad
&
\frac{1}{2}\|\mathbf{W}-\mathbf{U}\mathbf{V}^{\top}\|_F^2
+
\frac{\lambda}{2}
\left(
\|\mathbf{U}\|_F^2+\|\mathbf{V}\|_F^2
\right)
+
\mathcal{I}_{\mathcal{C}_{\mathbf{U}}}(\mathbf{Z}_{\mathbf{U}})
+
\mathcal{I}_{\mathcal{C}_{\mathbf{V}}}(\mathbf{Z}_{\mathbf{V}})
\\
\mathrm{s.t.}
\quad
&
\mathbf{U}=\mathbf{Z}_{\mathbf{U}},
\qquad
\mathbf{V}=\mathbf{Z}_{\mathbf{V}},
\end{aligned}
\label{eq:binary_factor_problem}
\end{equation}
where $\lambda\ge 0$ is the ridge regularization coefficient, and $\mathcal{I}_{\mathcal{C}_{\mathbf{U}}}$ and $\mathcal{I}_{\mathcal{C}_{\mathbf{V}}}$ are indicator functions for the structured proxy families used to initialize the latent binary factors.

\subsection{Spectral Control from Robust Preconditioning}

The stability of the initialization depends partly on controlling the scale of the preconditioned target $\mathbf{W}$. 
The robust diagonal estimator used in \name applies clipping and shrinkage to reduce the influence of outlier calibration statistics.

\begin{lemma}[Bounded preconditioners]
\label{lem:phi_bound}
Assume that the diagonal preconditioner entries are nonnegative and clipped to a maximum value $\tau_{\max}>0$. 
Then the diagonal preconditioners satisfy
\begin{equation}
    \|\widetilde{\mathbf{D}}_{\mathrm{out}}\|_2 \le \tau_{\max},
    \qquad
    \|\widetilde{\mathbf{D}}_{\mathrm{in}}\|_2 \le \tau_{\max}.
\end{equation}
\end{lemma}

\begin{proof}
For a diagonal matrix, the spectral norm is the maximum absolute diagonal entry. 
Since each diagonal entry of $\widetilde{\mathbf{D}}_{\mathrm{out}}$ and $\widetilde{\mathbf{D}}_{\mathrm{in}}$ is nonnegative and clipped to be at most $\tau_{\max}$, their spectral norms are bounded by $\tau_{\max}$.
\end{proof}

\begin{corollary}[Bounded preconditioned target]
\label{cor:w_bound}
The preconditioned target satisfies
\begin{equation}
    \|\mathbf{W}\|_2
    \le
    \tau_{\max}^2
    \|\mathbf{W}_{\mathrm{FP}}\|_2 .
\end{equation}
\end{corollary}

\begin{proof}
Using submultiplicativity of the spectral norm,
\begin{equation}
\begin{aligned}
    \|\mathbf{W}\|_2
    &=
    \left\|
    \widetilde{\mathbf{D}}_{\mathrm{out}}
    \mathbf{W}_{\mathrm{FP}}
    \widetilde{\mathbf{D}}_{\mathrm{in}}
    \right\|_2 \\
    &\le
    \|\widetilde{\mathbf{D}}_{\mathrm{out}}\|_2
    \|\mathbf{W}_{\mathrm{FP}}\|_2
    \|\widetilde{\mathbf{D}}_{\mathrm{in}}\|_2
    \le
    \tau_{\max}^2
    \|\mathbf{W}_{\mathrm{FP}}\|_2 .
\end{aligned}
\end{equation}
\end{proof}

This bound does not imply global convergence of the nonconvex factorization problem. 
However, it ensures that the target matrix passed to LB-ADMM has controlled spectral scale, which helps avoid extreme numerical ranges in the subsequent linear solves.

\subsection{LB-ADMM Updates}

We use the scaled-dual ADMM form with penalty parameter $\rho>0$ and scaled dual variables $\boldsymbol{\Lambda}_{\mathbf{U}}$ and $\boldsymbol{\Lambda}_{\mathbf{V}}$. 
Given $\mathbf{V}$, $\mathbf{Z}_{\mathbf{U}}$, and $\boldsymbol{\Lambda}_{\mathbf{U}}$, the continuous update for $\mathbf{U}$ solves
\begin{equation}
\min_{\mathbf{U}}
\quad
\frac{1}{2}\|\mathbf{W}-\mathbf{U}\mathbf{V}^{\top}\|_F^2
+
\frac{\lambda}{2}\|\mathbf{U}\|_F^2
+
\frac{\rho}{2}
\|\mathbf{U}-\mathbf{Z}_{\mathbf{U}}+\boldsymbol{\Lambda}_{\mathbf{U}}\|_F^2 .
\end{equation}
The corresponding normal equation is
\begin{equation}
    \mathbf{U}
    \left(
    \mathbf{V}^{\top}\mathbf{V}
    +
    (\rho+\lambda)\mathbf{I}
    \right)
    =
    \mathbf{W}\mathbf{V}
    +
    \rho
    \left(
    \mathbf{Z}_{\mathbf{U}}
    -
    \boldsymbol{\Lambda}_{\mathbf{U}}
    \right).
\end{equation}
Equivalently,
\begin{equation}
    \left(
    \mathbf{V}^{\top}\mathbf{V}
    +
    (\rho+\lambda)\mathbf{I}
    \right)
    \mathbf{U}^{\top}
    =
    \mathbf{V}^{\top}\mathbf{W}^{\top}
    +
    \rho
    \left(
    \mathbf{Z}_{\mathbf{U}}
    -
    \boldsymbol{\Lambda}_{\mathbf{U}}
    \right)^{\top}.
\end{equation}
The update for $\mathbf{V}$ is symmetric:
\begin{equation}
    \mathbf{V}
    \left(
    \mathbf{U}^{\top}\mathbf{U}
    +
    (\rho+\lambda)\mathbf{I}
    \right)
    =
    \mathbf{W}^{\top}\mathbf{U}
    +
    \rho
    \left(
    \mathbf{Z}_{\mathbf{V}}
    -
    \boldsymbol{\Lambda}_{\mathbf{V}}
    \right).
\end{equation}

The proxy variables are updated using the SVID operator:
\begin{equation}
    \mathbf{Z}_{\mathbf{U}}
    \leftarrow
    \mathrm{SVID}
    \left(
    \mathbf{U}+\boldsymbol{\Lambda}_{\mathbf{U}}
    \right),
    \qquad
    \mathbf{Z}_{\mathbf{V}}
    \leftarrow
    \mathrm{SVID}
    \left(
    \mathbf{V}+\boldsymbol{\Lambda}_{\mathbf{V}}
    \right).
\end{equation}
This step imposes the structured sign-value proxy used for latent binary initialization. 
Since the proxy families are nonconvex and discrete, the SVID step should be interpreted as a structured proxy update rather than a convex projection.

Finally, the scaled dual variables are updated as
\begin{equation}
    \boldsymbol{\Lambda}_{\mathbf{U}}
    \leftarrow
    \boldsymbol{\Lambda}_{\mathbf{U}}
    +
    \mathbf{U}
    -
    \mathbf{Z}_{\mathbf{U}},
    \qquad
    \boldsymbol{\Lambda}_{\mathbf{V}}
    \leftarrow
    \boldsymbol{\Lambda}_{\mathbf{V}}
    +
    \mathbf{V}
    -
    \mathbf{Z}_{\mathbf{V}}.
\end{equation}

\subsection{Well-Posed Continuous Updates}

\begin{lemma}[Positive definiteness of the continuous subproblems]
\label{lem:spd_unique}
For any $\rho>0$ and $\lambda\ge 0$, the system matrices
\begin{equation}
    \mathbf{H}_{\mathbf{U}}
    =
    \mathbf{V}^{\top}\mathbf{V}
    +
    (\rho+\lambda)\mathbf{I},
    \qquad
    \mathbf{H}_{\mathbf{V}}
    =
    \mathbf{U}^{\top}\mathbf{U}
    +
    (\rho+\lambda)\mathbf{I}
\end{equation}
are symmetric positive definite. 
Consequently, the continuous updates for $\mathbf{U}$ and $\mathbf{V}$ are strongly convex quadratic subproblems with unique solutions.
\end{lemma}

\begin{proof}
For any nonzero vector $\mathbf{a}$,
\begin{equation}
    \mathbf{a}^{\top}
    \left(
    \mathbf{V}^{\top}\mathbf{V}
    +
    (\rho+\lambda)\mathbf{I}
    \right)
    \mathbf{a}
    =
    \|\mathbf{V}\mathbf{a}\|_2^2
    +
    (\rho+\lambda)\|\mathbf{a}\|_2^2.
\end{equation}
Since $\rho>0$ and $\lambda\ge 0$, this quantity is strictly positive. 
Thus, $\mathbf{H}_{\mathbf{U}}$ is positive definite. 
The same argument applies to $\mathbf{H}_{\mathbf{V}}$.
The corresponding objectives are therefore strongly convex quadratics in the updated variable and admit unique minimizers.
\end{proof}

\begin{corollary}[Conditioning of the linear solves]
\label{cor:conditioning}
The condition number of the $\mathbf{U}$ update matrix satisfies
\begin{equation}
    \kappa(\mathbf{H}_{\mathbf{U}})
    =
    \frac{
    \sigma_{\max}(\mathbf{V})^2+\rho+\lambda
    }{
    \sigma_{\min}(\mathbf{V})^2+\rho+\lambda
    }
    \le
    1+
    \frac{\|\mathbf{V}\|_2^2}{\rho+\lambda}.
\end{equation}
An analogous bound holds for the $\mathbf{V}$ update matrix:
\begin{equation}
    \kappa(\mathbf{H}_{\mathbf{V}})
    \le
    1+
    \frac{\|\mathbf{U}\|_2^2}{\rho+\lambda}.
\end{equation}
\end{corollary}

\begin{proof}
The eigenvalues of $\mathbf{V}^{\top}\mathbf{V}$ are the squared singular values of $\mathbf{V}$. 
Adding $(\rho+\lambda)\mathbf{I}$ shifts every eigenvalue by $\rho+\lambda$, giving
\begin{equation}
    \kappa(\mathbf{H}_{\mathbf{U}})
    =
    \frac{
    \sigma_{\max}(\mathbf{V})^2+\rho+\lambda
    }{
    \sigma_{\min}(\mathbf{V})^2+\rho+\lambda
    }.
\end{equation}
Since $\sigma_{\min}(\mathbf{V})^2\ge 0$, we have
\begin{equation}
    \kappa(\mathbf{H}_{\mathbf{U}})
    \le
    \frac{
    \sigma_{\max}(\mathbf{V})^2+\rho+\lambda
    }{
    \rho+\lambda
    }
    =
    1+
    \frac{\|\mathbf{V}\|_2^2}{\rho+\lambda}.
\end{equation}
The proof for $\mathbf{H}_{\mathbf{V}}$ is identical.
\end{proof}

This result explains the roles of the ADMM penalty and ridge regularization. 
The diagonal shift $(\rho+\lambda)\mathbf{I}$ guarantees positive definiteness and prevents singular continuous updates, even when one factor is rank-deficient. 
Increasing $\rho$ strengthens the coupling between the continuous factors and the structured proxy variables, while increasing $\lambda$ further controls the factor magnitudes. 
Both terms improve the worst-case conditioning of the Cholesky solves through the same positive diagonal shift.

\subsection{Role of the SVID Proxy Step}

The SVID update imposes the sign-value structure required for latent binary initialization. 
Because the proxy families $\mathcal{C}_{\mathbf{U}}$ and $\mathcal{C}_{\mathbf{V}}$ are nonconvex and discrete, standard convex ADMM convergence guarantees do not directly apply. 
Accordingly, we use LB-ADMM as a stable initialization procedure rather than as a globally convergent solver for the discrete factorization problem.

The optimization remains well-posed in the sense required for initialization: each continuous update is a unique SPD solve, and each proxy update maps the continuous factor back to the structured family used to initialize binary factors. 
The subsequent magnitude balancing and block reconstruction stages further refine these initialized factors using calibration data.

\paragraph{Summary.}
The above analysis establishes three stability properties of LB-ADMM initialization. 
First, robust preconditioning bounds the spectral scale of the target matrix. 
Second, the continuous ADMM updates are well-posed SPD linear systems with unique solutions. 
Third, the penalty and ridge terms help control the conditioning of these solves. 
These properties do not constitute a global convergence proof for the nonconvex discrete problem, but they justify LB-ADMM as a stable and computationally efficient initialization procedure for low-rank binary quantization.

\section{Implementation Details}

All compression experiments for \name were conducted on 1 NVIDIA H100 (80 GB), and we utilize unified hyperparameters when compressing models with \name.
When tuning pre-factorized parameters to absorb quantization error, we used a learning rate of 1e-4 and a batch size of 4.
For tuning factorized parameters (low-rank, latent binary and full-precision scales), we used a unified learning rate of 1e-5 and a batch size of 1.
For global scale reconstruction, we used a learning rate of 1e-6 and a batch size of 1.
Pre-factorized, factorized, and global tuning stages all consist of 8 epochs and utilize a cosine learning rate scheduler.
We employed a linear ADMM penalty scheduler for 400 factorization steps, for each weight matrix across all models.
For calibration data, we used 128 samples with a sequence length of 2048 from the WikiText-2 dataset \cite{wikitext}, and used a random seed value of 0 for data selection.
For all experiments, we used \texttt{torch=2.6.0}, \texttt{transformers=4.51.3}, \texttt{datasets=4.0.0}, \texttt{lm\_eval=0.4.9}, and CUDA 12.4.
To derive the activation- and gradient-based diagonal preconditioners, we utilized gradient checkpointing and used a memory-efficient implementation of the cross-entropy function \cite{cce}.
During block reconstruction, we employed a weighted MSE function, utilized in previous quantization works \cite{dbf, kim2025guidedquant}.
Official open-source implementations or quantized models were used to evaluate baselines for binary PTQ \cite{bi_llm, stb_llm, arb_llm, hb_llm}, vector quantization \cite{aqlm, pv_tuning, qtip}, and low-rank binary QAT \cite{dbf, littlebit}.

Since none of the binary PTQ baselines (BiLLM \cite{bi_llm}, STBLLM \cite{stb_llm}, ARB-LLM \cite{arb_llm}, HBLLM \cite{hb_llm}) compress quantized models into memory-efficient formats, we utilize the main text and appendices of such methods to calculate both effective bits and model checkpoint sizes \cite{bi_llm, arb_llm, hb_llm}.
Further details on effective bit calculation and model checkpoint sizes can be found in \autoref{appendix:model_size}.

\newpage
\section{Further Ablations}
\label{appendix:additional_exp}

\subsection{Different Data Budgets}

We test with different calibration data budgets for the block and model reconstruction stages.
As in \autoref{tab:different_calib}, utilizing more data during block reconstruction leads to greater performance gains.

\begin{table}[H]
    \centering
    \caption{
    Utilizing different data budgets for the block and model reconstruction stages of \name, when compressing Llama-2-7B to 1-bit.
    }
    \label{tab:different_calib}
    \small
    \renewcommand{\arraystretch}{1.2}
    \setlength{\tabcolsep}{5pt}
    
    \begin{tabular}{cccccc}
        \toprule
        \multirow{2}{*}[-0.3em]{\textbf{\shortstack{Block Recon.\\Samples}}} & \multicolumn{5}{c}{\textbf{Model Recon. Samples}} \\
        \cmidrule(lr){2-6}
         & 32 & 64 & 128 & 256 & 512 \\
        \midrule
        32  & 14.16 & 13.24 & 12.72 & 12.22 & 11.91 \\
        64  & 12.13 & 11.66 & 11.23 & 10.94 & 10.69 \\
        128 & 10.56 & 10.40 & 10.35 & 10.06 & 9.89  \\
        256 & 10.25 & 10.24 & 10.17 & 9.77  & 9.60  \\
        512 & 9.24  & 9.16  & 9.13  & 9.07  & \textbf{9.07}  \\
        \bottomrule
    \end{tabular}
\end{table}

\subsection{Different Calibration Datasets}

We test how the calibration dataset impacts performance.
\autoref{tab:calibration_diversity} shows that increasing calibration data diversity mitigates overfitting and can improve zero-shot performance.

\begin{table*}[h!]
\centering
\caption{
Effect of calibration dataset composition on 1-bit quantization for instruction-tuned Qwen3-4B and Llama-2-7B.
We vary the number of calibration samples drawn from C4 and WikiText-2 (WT2), while keeping the total fixed at 128.
}
\label{tab:calibration_diversity}
\vspace{0.8em}
\small
\setlength{\tabcolsep}{6pt}
\begin{tabular}{cc | ccc | ccc}
\toprule
\multicolumn{2}{c|}{\textbf{Samples}} & \multicolumn{3}{c|}{\textbf{Qwen3-4B, 1-bit}} & \multicolumn{3}{c}{\textbf{Llama-2-7B, 1-bit}} \\
\textbf{C4} & \textbf{WT2} & \textbf{WT2 PPL} ($\downarrow$) & \textbf{C4 PPL} ($\downarrow$) & \textbf{Zero-shot} ($\uparrow$) & \textbf{WT2 PPL} ($\downarrow$) & \textbf{C4 PPL} ($\downarrow$) & \textbf{Zero-shot} ($\uparrow$) \\
\midrule
0   & 128 & \textbf{21.25} & 58.28 & 46.53 & \textbf{10.34} & 23.23 & 48.68 \\
32  & 96  & 22.11 & 41.24 & 48.66 & 10.61 & 17.92 & 49.06 \\
64  & 64  & 22.66 & 39.22 & 47.08 & 10.74 & 17.15 & 50.08 \\
96  & 32  & 23.92 & 36.41 & 50.19 & 11.42 & \textbf{16.25} & \textbf{50.87} \\
128 & 0   & 38.96 & \textbf{35.23} & \textbf{50.22} & 17.99 & 16.31 & 48.70 \\
\bottomrule
\end{tabular}
\end{table*}

\subsection{Analysis of Latent Weight Dynamics}
\label{appendix:latent_dynamics}

To validate the efficacy of the Factorized Component Refinement phase (Step 3), we analyze the trajectory of the continuous latent variables, $\mathcal{U}$ and $\mathcal{V}$, before and after fine-tuning.
\autoref{fig:latent_mobility} visualizes the distribution shifts and sign flip ratios for all linear layers within the first transformer block of Llama-3.2-1B.
The abscissa represents the magnitude of the latent weights initialized by LB-ADMM (Step 2), while the ordinate in the right-hand panels denotes the magnitude of change after refinement (Step 3).

\paragraph{Stability of Initialization.}
As illustrated in \autoref{fig:latent_mobility}, a predominant proportion of latent weights retain their original signs throughout the refinement process.
The sign flip ratio remains consistently low, ranging from 0.47\% in the \texttt{k\_proj} layer to 6.82\% in the \texttt{gate\_proj} layer.
This stability indicates that the LB-ADMM initialization establishes a parameter configuration proximate to a local optimum, thereby mitigating the necessity for substantial updates during the fine-tuning phase.

\paragraph{Refinement of Boundary Weights.}
Despite the low flip ratio, the refinement step functions as a critical margin maximization process.
The interaction density plots reveal an inverse correlation between the initial magnitude and the degree of change.
Specifically, weights with initial magnitudes near zero, corresponding to the decision boundary, exhibit the highest mobility and likelihood of sign flipping.
This behavior suggests that the refinement phase selectively rectifies the signs of ``ambiguous'' weights near the zero boundary while preserving the confident decisions established during the initialization.
Consequently, this targeted adjustment compensates for discretization errors introduced during the initial factorization.

\begin{figure*}[h!]
    \centering
    \begin{subfigure}[b]{0.49\textwidth}
        \centering
        \includegraphics[width=\linewidth]{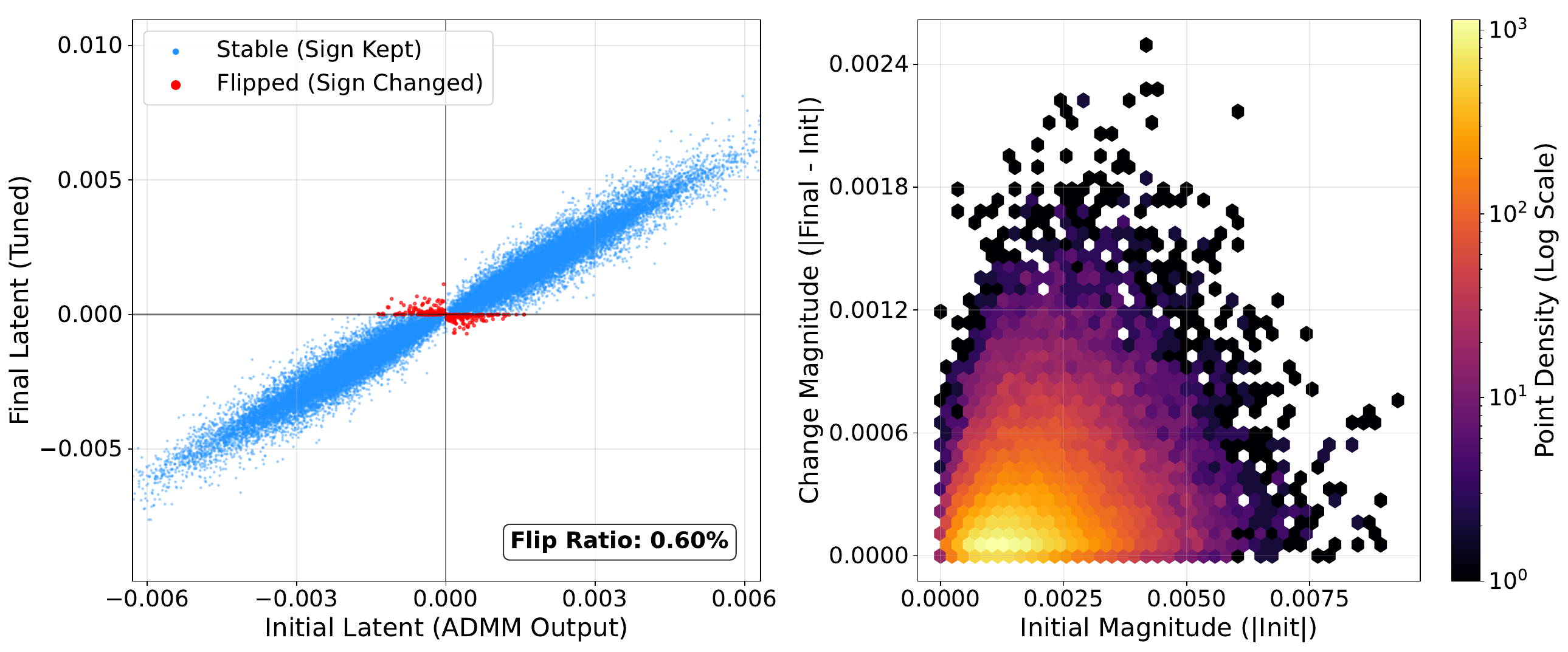}
        \caption{\texttt{self\_attn.q\_proj} (Flip: 0.60\%)}
        \label{fig:mobility_q}
    \end{subfigure}
    \hfill
    \begin{subfigure}[b]{0.49\textwidth}
        \centering
        \includegraphics[width=\linewidth]{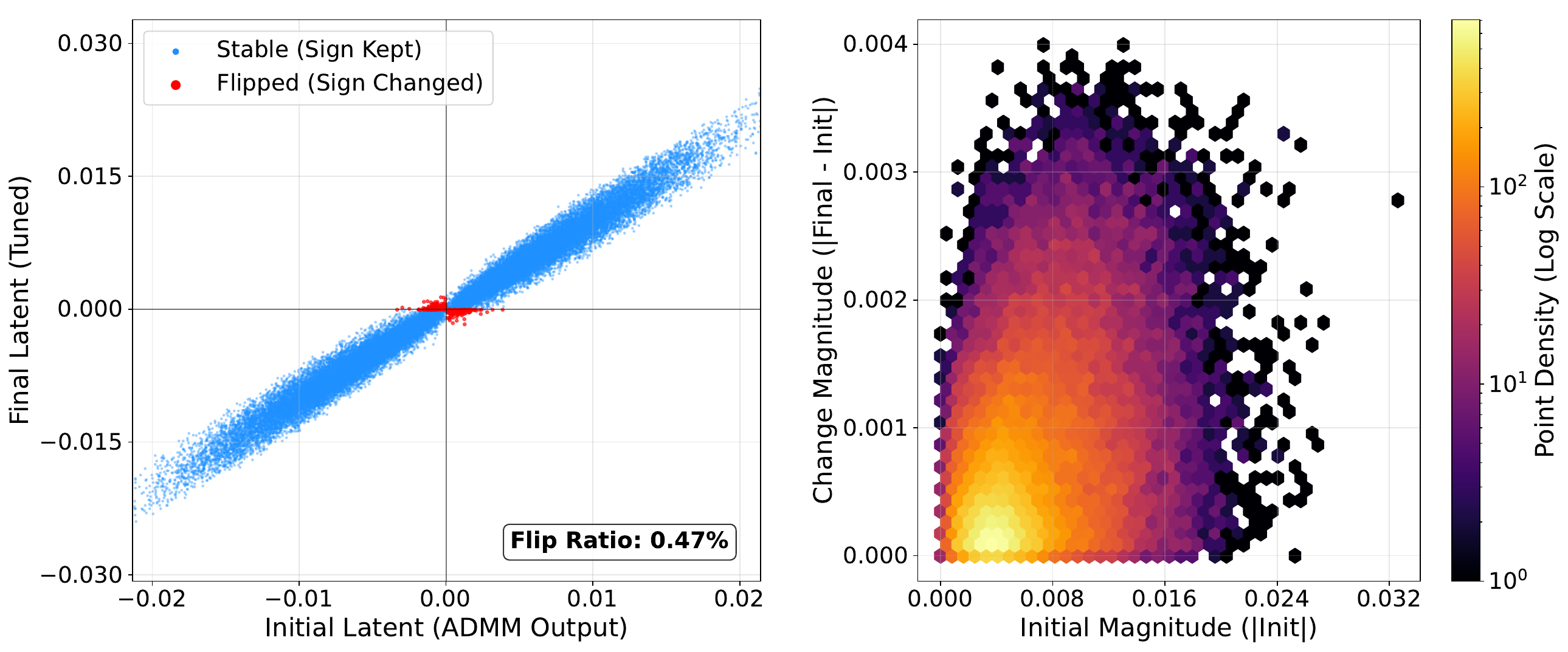}
        \caption{\texttt{self\_attn.k\_proj} (Flip: 0.47\%)}
        \label{fig:mobility_k}
    \end{subfigure}
    
    \begin{subfigure}[b]{0.49\textwidth}
        \centering
        \includegraphics[width=\linewidth]{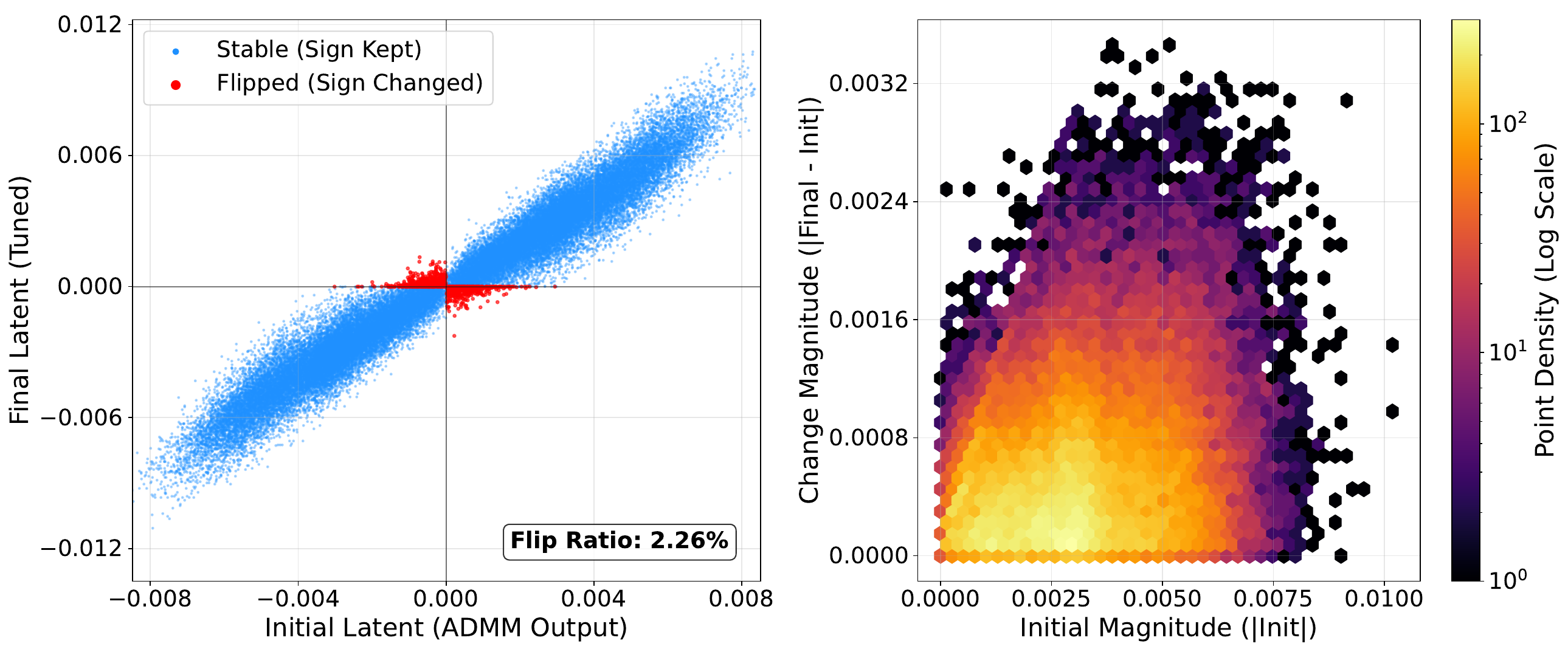}
        \caption{\texttt{self\_attn.v\_proj} (Flip: 2.26\%)}
        \label{fig:mobility_v}
    \end{subfigure}
    \hfill
    \begin{subfigure}[b]{0.49\textwidth}
        \centering
        \includegraphics[width=\linewidth]{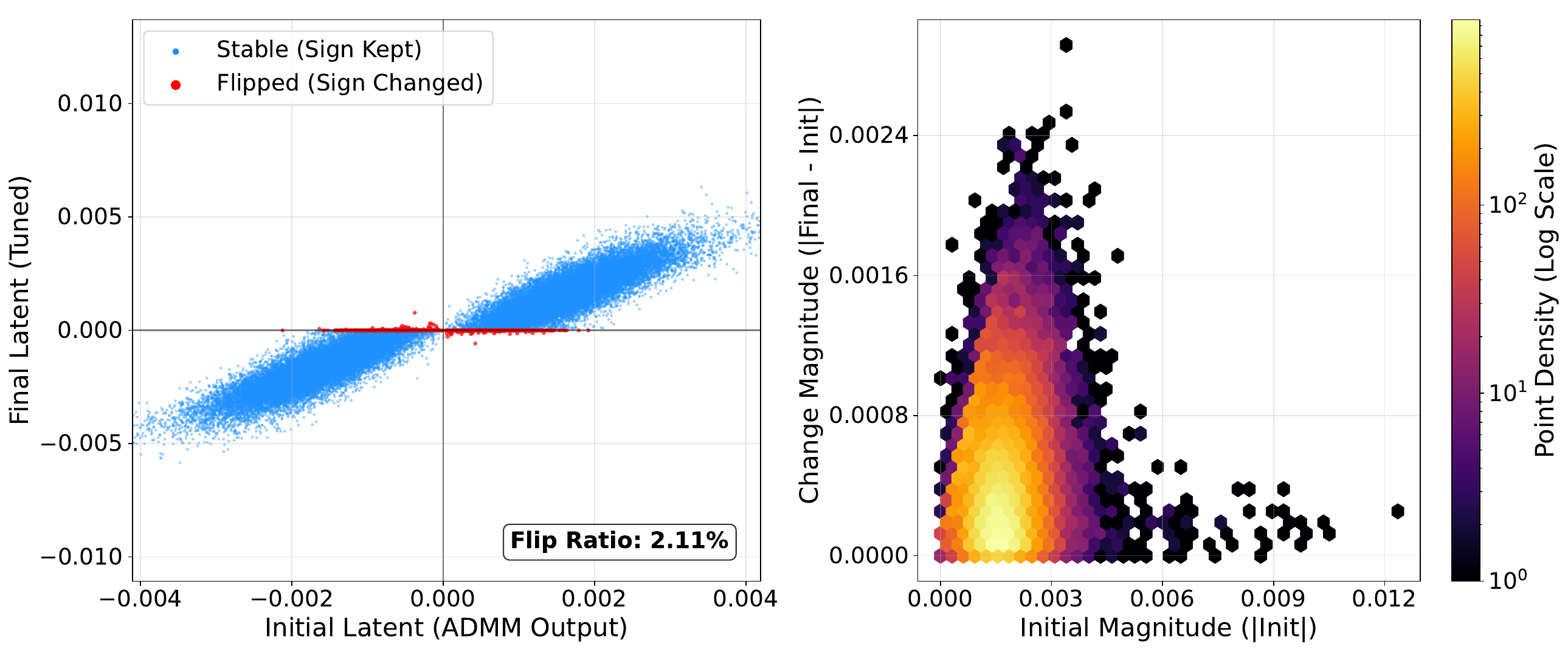}
        \caption{\texttt{self\_attn.o\_proj} (Flip: 2.11\%)}
        \label{fig:mobility_o}
    \end{subfigure}
    
    \begin{subfigure}[b]{0.49\textwidth}
        \centering
        \includegraphics[width=\linewidth]{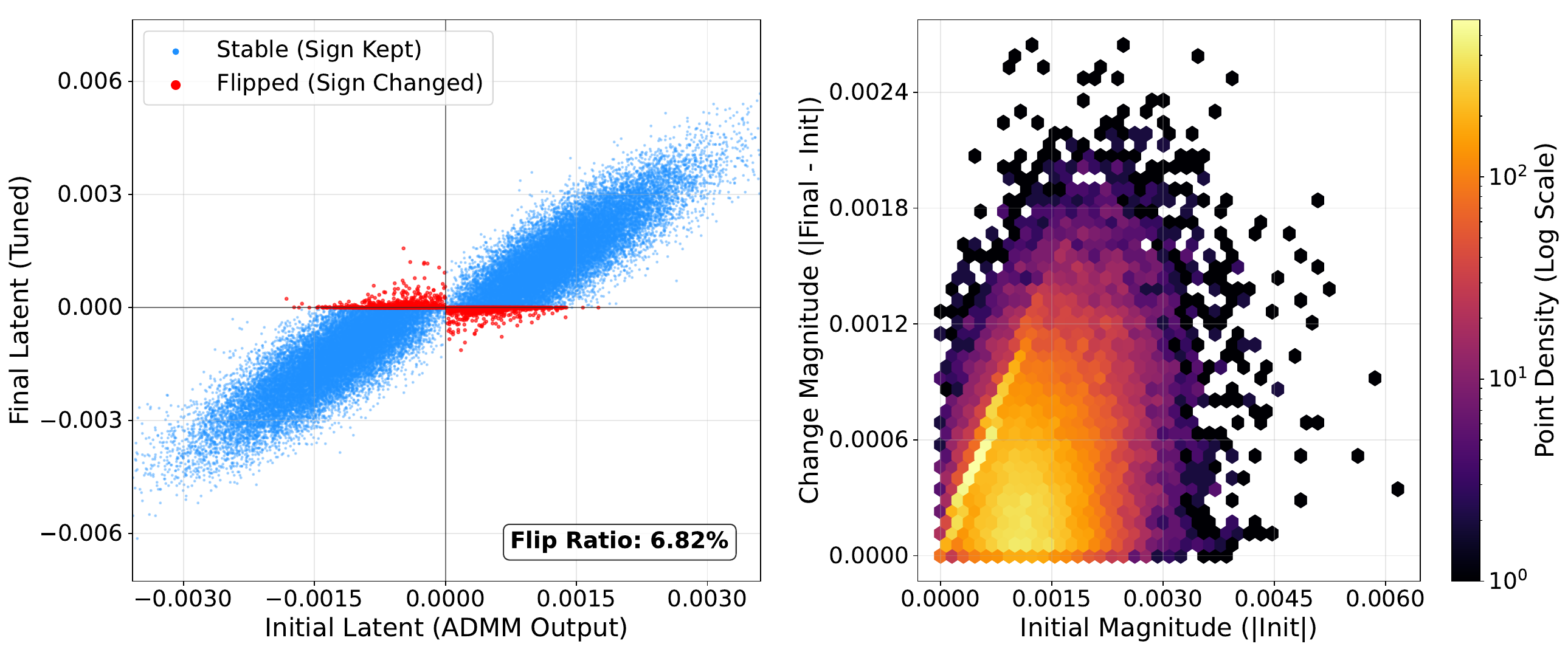}
        \caption{\texttt{mlp.gate\_proj} (Flip: 6.82\%)}
        \label{fig:mobility_gate}
    \end{subfigure}
    \hfill
    \begin{subfigure}[b]{0.49\textwidth}
        \centering
        \includegraphics[width=\linewidth]{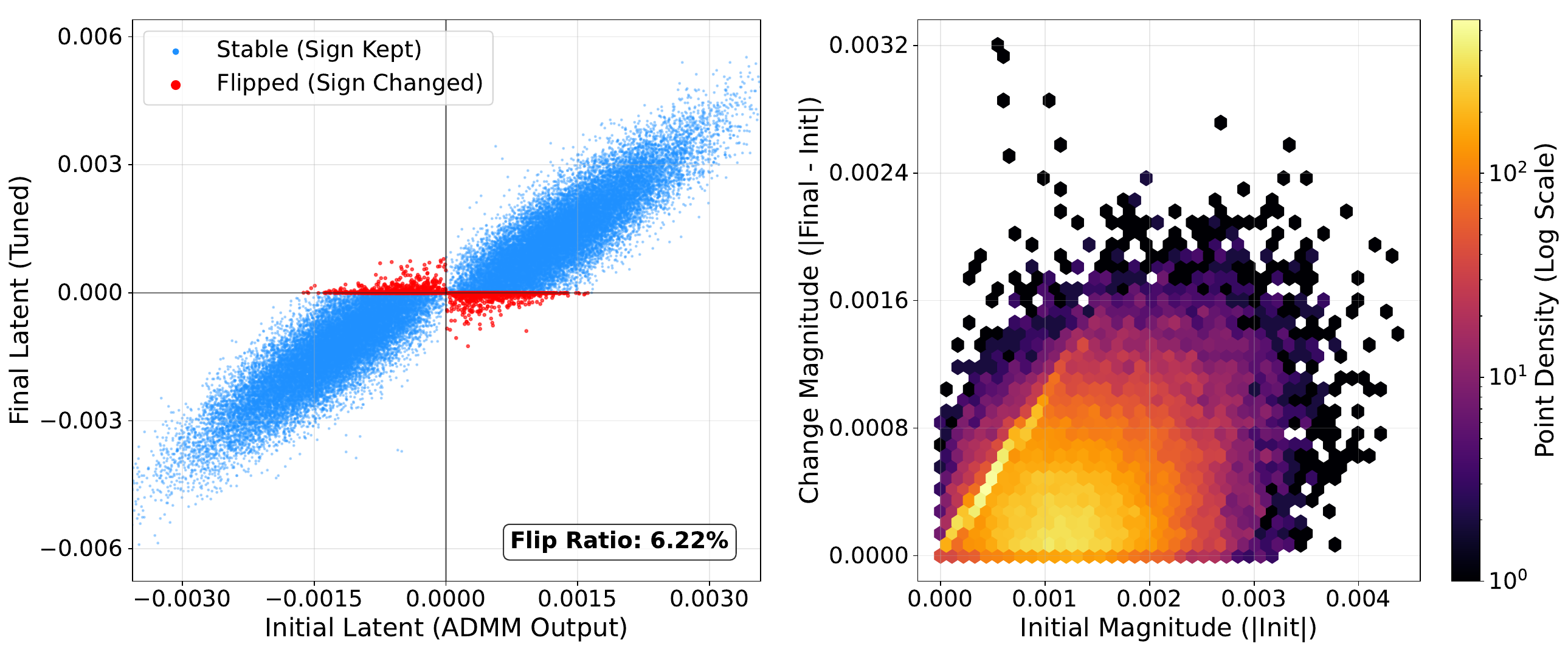}
        \caption{\texttt{mlp.up\_proj} (Flip: 6.22\%)}
        \label{fig:mobility_up}
    \end{subfigure}
    
    \begin{subfigure}[b]{0.49\textwidth}
        \centering
        \includegraphics[width=\linewidth]{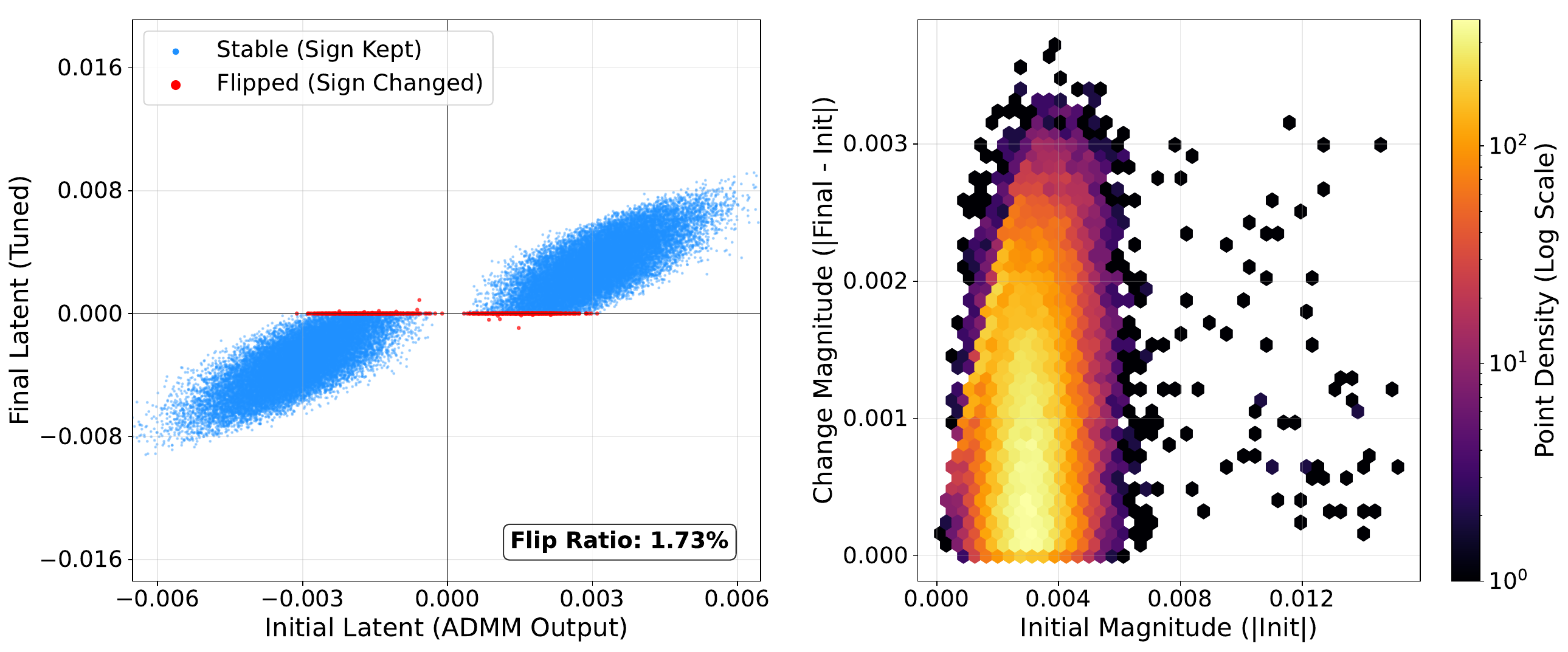}
        \caption{\texttt{mlp.down\_proj} (Flip: 1.73\%)}
        \label{fig:mobility_down}
    \end{subfigure}
    
    \caption{
    Visualization of latent variable dynamics between the initialization (LB-ADMM) and refinement (STE Tuning) phases for Llama-3.2-1B (Block 0).
    Blue points denote weights that retained their sign, while red points denote sign flips.
    The density plots in each panel illustrate that sign flips and large magnitude updates are concentrated around weights with near-zero initial magnitude.
    This demonstrates that the refinement step selectively optimizes decision boundaries for features with high uncertainty.
    }
    \label{fig:latent_mobility}
\end{figure*}

\clearpage

\subsection{ADMM Ablation Experiments}

We ablate two ADMM optimization choices on block 0 of Gemma 4 31B: the number of outer iterations and the penalty scheduling strategy. The results show a clear speed-accuracy trade-off. Reducing the number of ADMM iterations gives faster convergence but leads to a higher final reconstruction error.

\begin{figure}[H]
    \centering
    \captionsetup{font=small,skip=1pt}
    \captionsetup[subfigure]{font=footnotesize,skip=1pt}

    \begin{subfigure}{\linewidth}
        \centering
        \includegraphics[
            width=0.90\linewidth,
            height=0.50\textheight,
            keepaspectratio
        ]{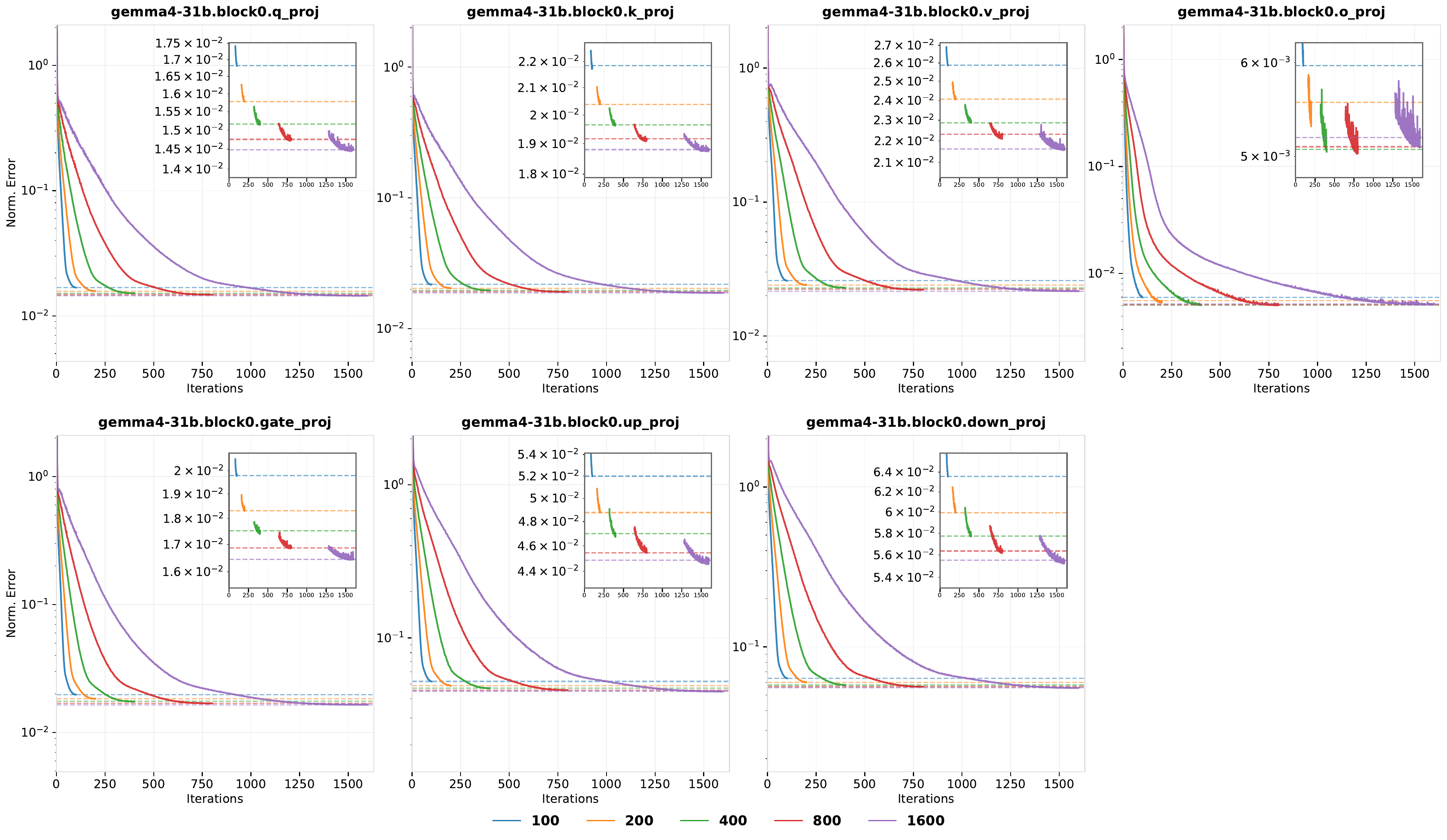}
        \caption{
        Effect of ADMM outer iterations. Fewer iterations reduce optimization cost but produce higher final reconstruction error.
        }
        \label{fig:admm_outer_iters}
    \end{subfigure}
    
    \vspace{1.0em}
    
    \begin{subfigure}{\linewidth}
        \centering
        \includegraphics[
            width=0.90\linewidth,
            height=0.50\textheight,
            keepaspectratio
        ]{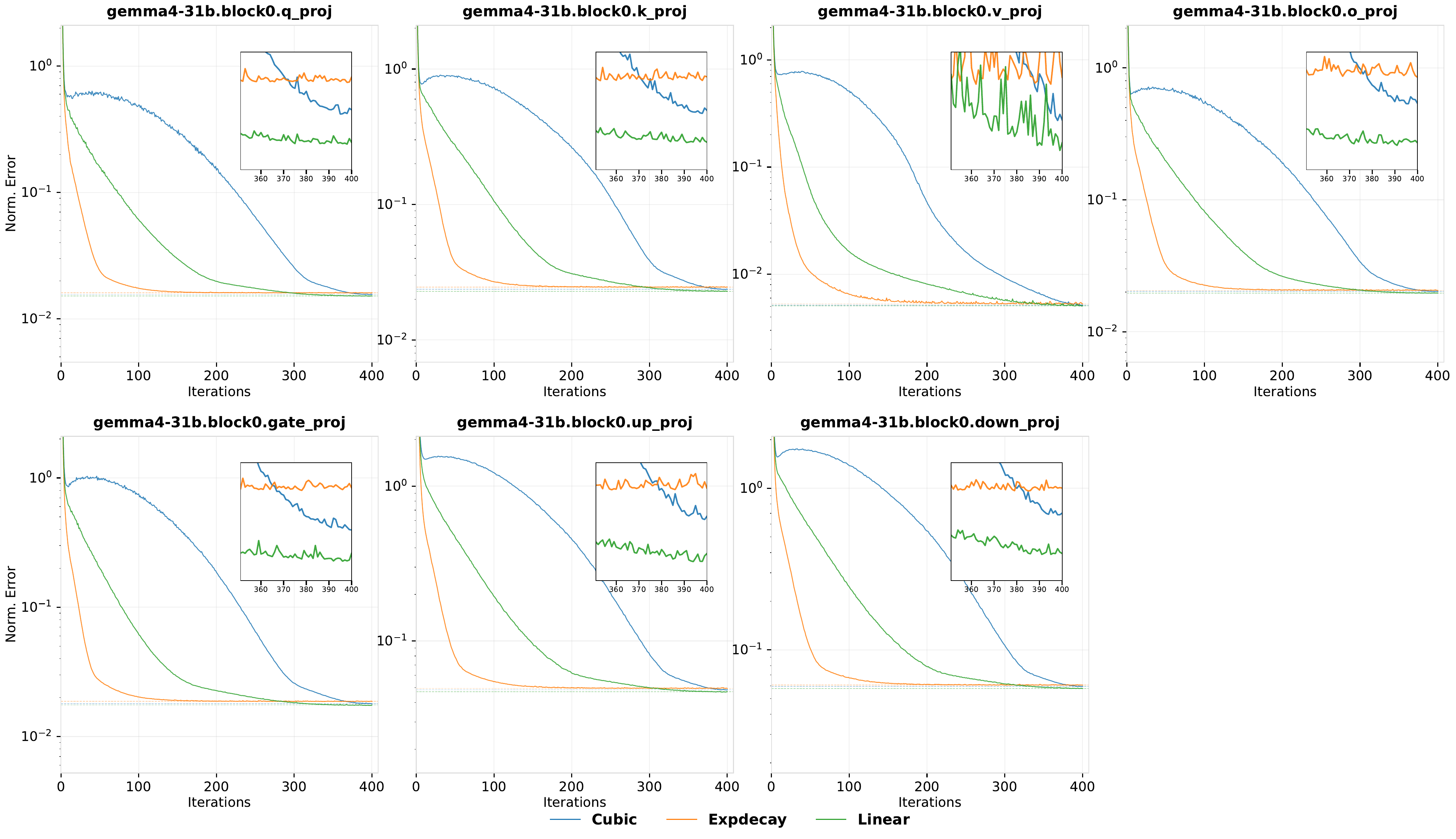}
        \caption{
        Effect of ADMM penalty scheduling. Linear scheduling converges more slowly but achieves lower final error than more aggressive schedules.
        }
        \label{fig:admm_penalty_scheduler}
    \end{subfigure}


    \caption{
    ADMM ablations on Gemma-4-31B block 0. The number of outer iterations controls the optimization cost, while the penalty scheduler controls the convergence profile and final reconstruction quality.
    }
    \label{fig:admm_ablation_experiments}
\end{figure}

\newpage
\section{Inference Ablations}
\label{appendix:kernel_impl}

\subsection{Kernel Benchmarking Details}
We benchmark custom CUDA kernels using torch.compile from torch 2.6.0 \cite{pytorch} and \texttt{StaticCache} from the transformers library \cite{transformers}, with CUDA 12.4.
The decoding script is based on the open-source implementation for QTIP \cite{qtip}, and is used for all kernel evaluations.
For GEMV decoding (batch size = 1), we vary the number of output tokens.   
For GEMM inference, we evaluate performance under increasing batch sizes.
We fix the input tokens to 128, output tokens for batched inference to 512, temperature to 0.8, and the top-k value to 32.
We utilize the open-source \href{https://github.com/ml-energy/zeus}{\texttt{ml-energy/zeus}} library for all energy measurements.

We benchmark our kernels on 4 different GPUs, as listed in \autoref{tab:hardware_specs}.
Notably, we test on high-end GPUs, a consumer GPU, and an edge device with \textit{no NVIDIA Tensor Cores}, to test the decoding and efficiency limits of our custom binary CUDA kernels.

\begin{table}[htbp]  
\caption{Hardware specifications of devices we benchmark our custom GPU kernels on.}    
\label{tab:hardware_specs}  
\centering  
\begin{tabular}{lccccc}
        \toprule  
        \multirow{2}{*}{Device} & \multicolumn{3}{c}{Memory} & \multicolumn{2}{c}{Compute} \\      
        \cmidrule(lr){2-4} \cmidrule(lr){5-6}    
        & GPU Memory (GB) & Type & Bandwidth (GB/s) & CUDA Cores & Tensor Cores \\   
        \midrule  
        NVIDIA Jetson TX2   & 8     & LPDDR4    & 59.7  & 256     & 0     \\  
        NVIDIA RTX 3050     & 8     & GDDR6     & 224   & 2,560     & 80    \\  
        NVIDIA A100 SXM     & 80    & HBM2e     & 2039  & 6,912     & 432   \\  
        NVIDIA H100 PCIe    & 80    & HBM2e     & 2000  & 14,592    & 456   \\  
        \bottomrule    
    \end{tabular}
\end{table}  

\subsection{Binary GEMV Inference}

\begin{figure}[h]
    \centering
     \includegraphics[width=0.8\linewidth]{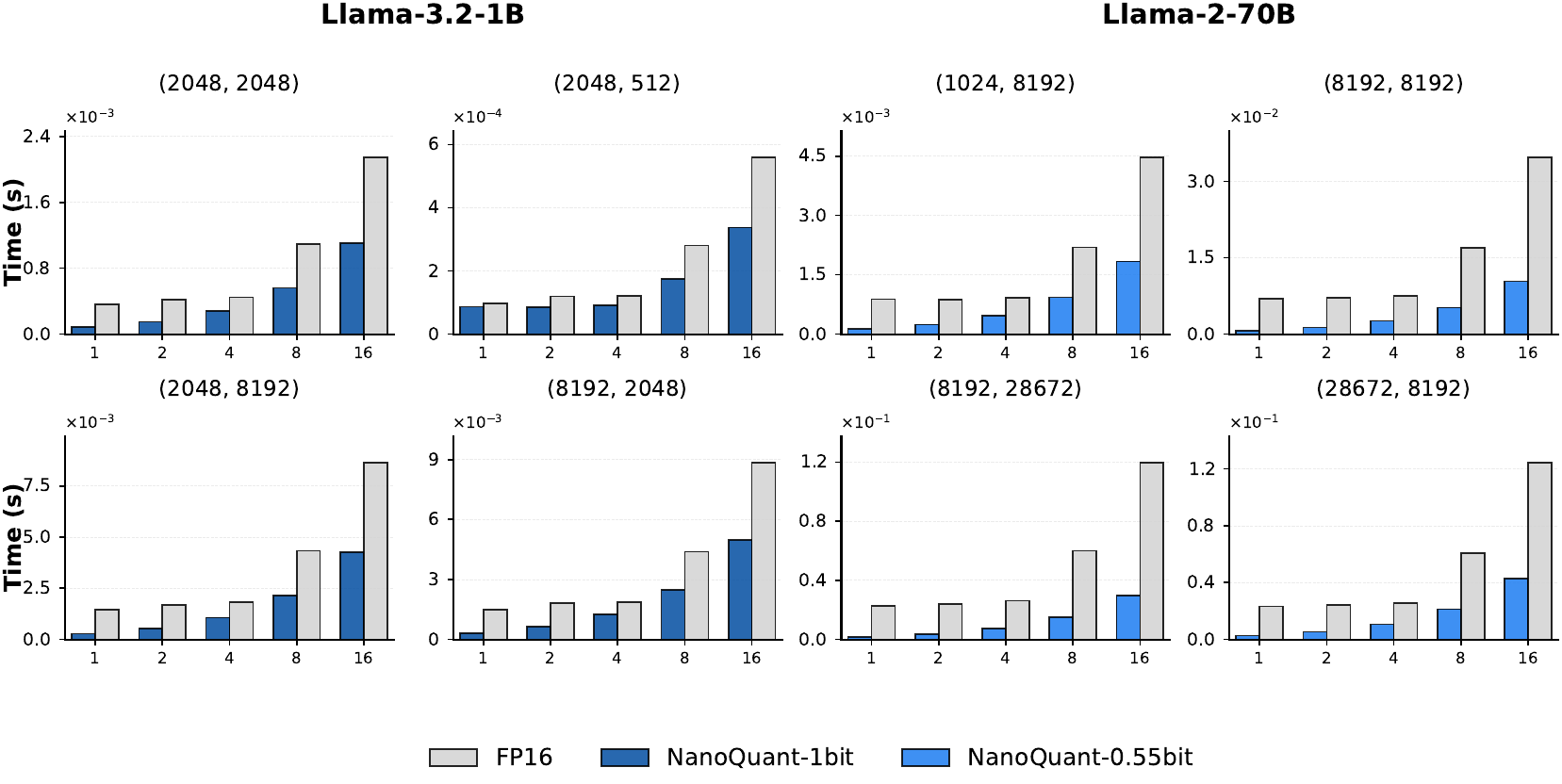}
    \caption{
    On the NVIDIA Jetson TX2, our custom GEMV kernels show significantly faster inference speeds than PyTorch FP16 for various matrix shapes, even for multiple vector batches.
    }
    \label{fig:jetson_gemv}
\end{figure}

\begin{table*}[h]
\centering
\caption{
Throughput (tokens/s) and peak memory (GB) for varying sequence lengths, for Llama-2 models compressed to 0.55 bits with \name.
Extreme compression with \name enables fast and memory-efficient inference on an NVIDIA RTX 3050 (8 GB).
}
\label{tab:llama2_tps_mem_vs_seqlen}
\setlength{\tabcolsep}{4pt}
\renewcommand{\arraystretch}{1.15}
\begin{tabular}{cccccccc}
\toprule
\multirow{2}{*}{Model} & \multirow{2}{*}{Metric} & \multicolumn{6}{c}{Sequence Length} \\
\cmidrule(lr){3-8}
 & & 32 & 64 & 128 & 256 & 512 & 1024 \\
\midrule
\multirow{2}{*}{Llama-2-7B}  & Tokens/s        & 134.10 & 133.40 & 127.04 & 122.52 & 108.44 & 86.27 \\
                            & Peak Mem (GB)    & 1.07 & 1.09 & 1.12 & 1.19 & 1.32 & 1.59 \\
\midrule
\multirow{2}{*}{Llama-2-13B} & Tokens/s        & 83.83 & 83.35 & 81.43 & 75.32 & 65.31 & 51.63 \\
                            & Peak Mem (GB)    & 1.70 & 1.73 & 1.78 & 1.88 & 2.09 & 2.57 \\
\midrule
\multirow{2}{*}{Llama-2-70B} & Tokens/s        & 20.11 & 19.74 & 19.18 & 17.68 & 15.37 & 12.13 \\
                            & Peak Mem (GB)    & 5.86 & 5.87 & 5.89 & 5.93 & 6.02 & 6.20 \\
\bottomrule
\end{tabular}
\end{table*}

\paragraph{GEMV CUDA kernel details.}

The GEMV kernel implements a two‑stage, 1‑bit quantized matrix–vector multiplication for \texttt{float16} and \texttt{bfloat16} tensors. In each stage, the input (or intermediate) vector is multiplied by a weight matrix whose signs are stored as a packed bitfield: each weight occupies a single bit in a \texttt{uint32} array. Because the binary matrices are low‑rank, the effective reduction in memory traffic is typically less than the theoretical 16×, but the packing still yields a substantial bandwidth saving. During execution the bits are unpacked on‑the‑fly with a lightweight mask operation, after which fused‑multiply‑add (FMA) is performed using vectorized \texttt{float16} or \texttt{bfloat16} intrinsics that process two low‑precision values per instruction. Per‑column scaling factors are incorporated directly into the FMA, and an optional per‑row scaling is applied before writing the intermediate or final result. The kernel does not invoke Tensor‑Core WMMA operations; instead it relies on standard FP16 arithmetic, without using NVIDIA Tensor Cores.

\paragraph{Results.}
We first benchmark our binary GEMV CUDA kernels against PyTorch BF16 and 2 state-of-the-art vector quantization methods, QTIP \cite{qtip} and AQLM \cite{aqlm}, on a single NVIDIA H100 GPU, as shown in \autoref{fig:h100_gemv_nano_vs_vq}.
To fully encompass the performance of the kernels, we measure the throughput (tokens per second), peak allocated GPU memory, and average energy per token during the single batch, end-to-end decoding process.
We utilize open-source models provided by QTIP and AQLM, and utilize models from the Llama-2 \cite{llama2} and Llama-3 \cite{llama3} families.

Next, we benchmark our binary GEMV CUDA kernels on an NVIDIA Jetson TX2, which does not have any NVIDIA Tensor Cores, as in \autoref{tab:hardware_specs}.
We find that the extreme compression capability of \name enables up to a $12.2\times$ speedup in inference throughput, compared to PyTorch FP16, as shown in \autoref{fig:jetson_gemv}.

\subsection{Binary GEMM Inference}
\paragraph{GEMM Kernel Details}
\label{appendix:gemm_details}

\begin{figure}[h]
    \centering
     \includegraphics[width=0.6\linewidth]{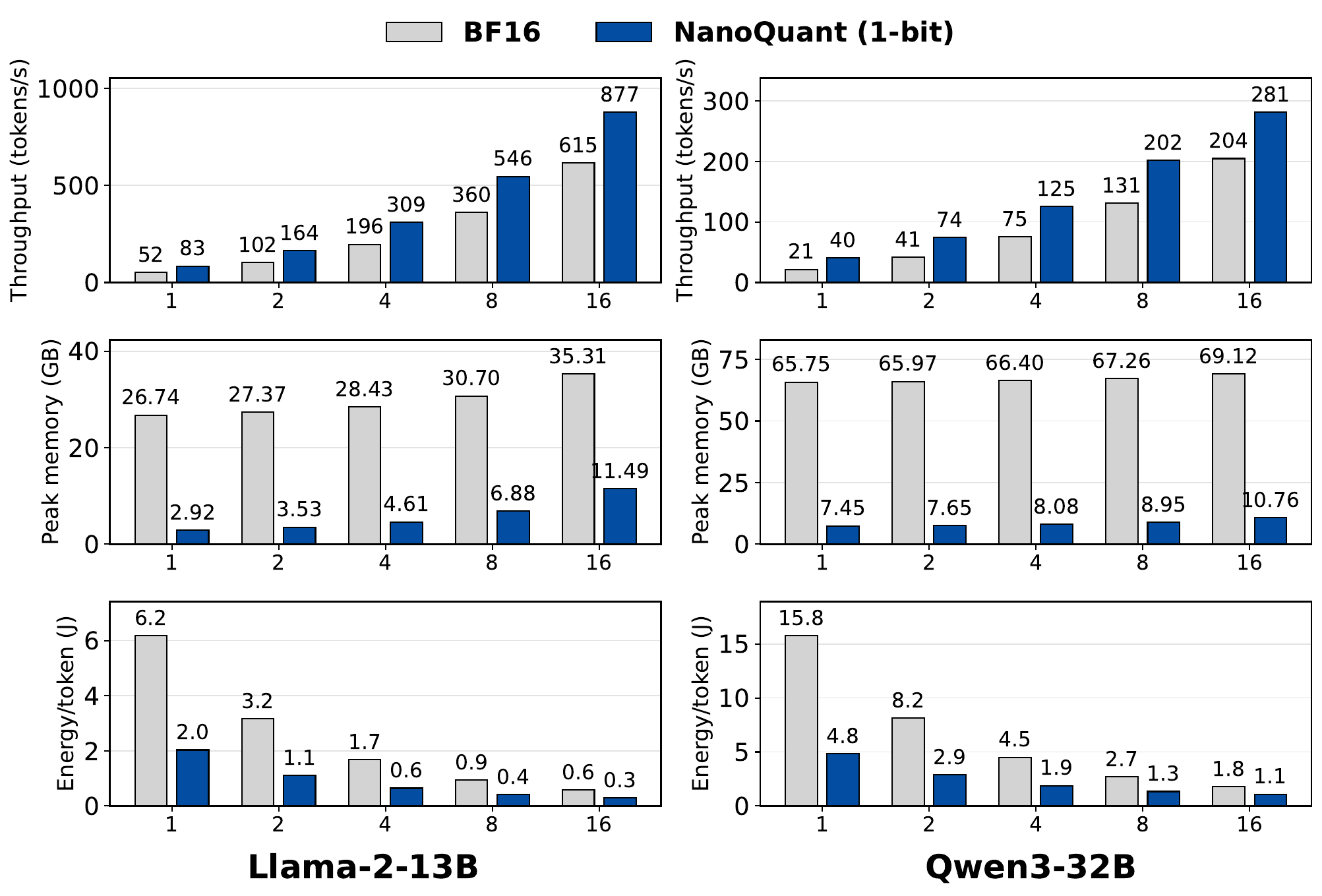}
    \caption{Custom GEMM kernels for \name achieve competitive batched inference performance with \texttt{BF16} PyTorch on a single NVIDIA A100 (80 GB) GPU.}
    \label{fig:a100_inference}
\end{figure}

Our GEMM kernel is a highly optimized CUDA GEMM implementation specifically designed for efficient low-rank binary matrix multiplication in quantized neural networks.
We base our implementation on the Marlin GEMM kernel \cite{marlin}, which leverages NVIDIA Tensor Cores for matrix multiplication through inline PTX assembly, processing matrix (\eg $16 \times 8 \times 16$) tiles with \texttt{mma.sync} operations.
The kernel employs a multi-stage pipeline (default 4 stages) with asynchronous memory operations (\texttt{cp.async}) to overlap data transfers with computation, effectively hiding memory latency.
It efficiently handles binary matrices by packing multiple 1-bit values into 32-bit words and using bit manipulation for fast dequantization, while maintaining computation in FP16/BF16 precision for accuracy.

While GEMV excels in low-batch, memory-bound scenarios, large-scale deployments benefit from batched GEMM operations that saturate tensor core throughput.
Therefore, Binary GEMM kernels are necessary for compute-bound LLM serving operations, especially for datacenter GPUs to fully utilize matrix-multiplication computation units, such as NVIDIA Tensor Cores.
The pipelined execution keeps compute units busy by overlapping the data loading, computation, and storing phases, while the warp-level parallelism with optimized thread scheduling maximizes GPU utilization.
These optimizations result in high arithmetic intensity and efficient use of hardware resources, making the kernel particularly well-suited for quantized neural network inference where binary low-rank matrices significantly decrease memory requirements and bandwidth usage without sacrificing model accuracy.

\subsection{Binary GPU Kernel Performance}

We compare the decoding performance of our binary GPU kernels with GemLite \cite{gemlite}, a state-of-the-art GPU kernel library that supports 1-bit kernels. As shown in \autoref{fig:h100_gemv_gemlite} and \autoref{fig:a100_gemv_gemlite}, our custom binary GEMV kernels outperform both BF16 and GemLite, with noNVIDIA Tensor Core free decoding operations.

\begin{figure*}[h]
    \centering
     \includegraphics[width=1.0\linewidth]{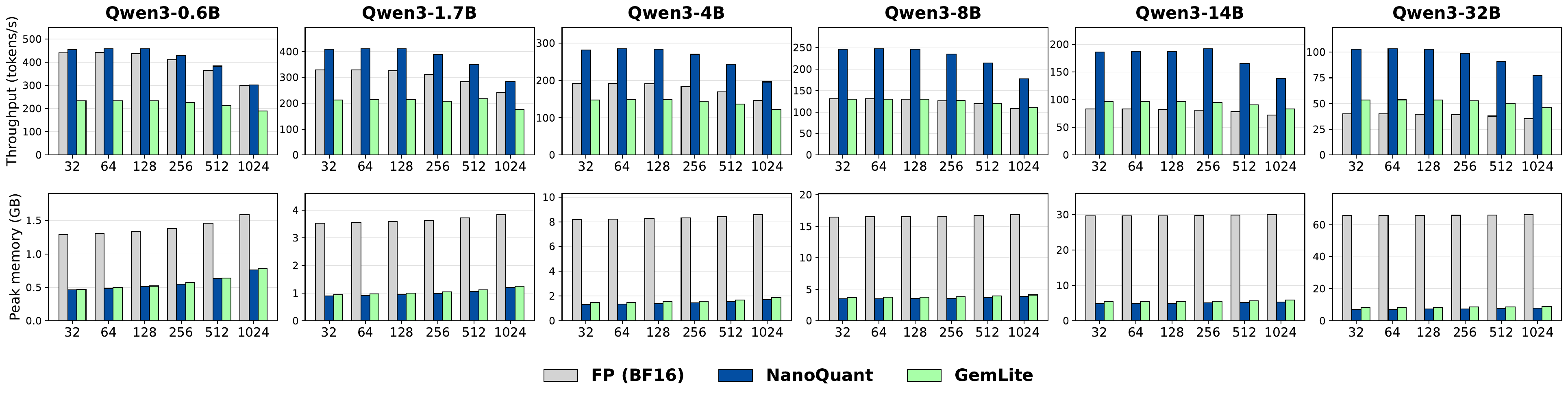}
    \caption{
    LLM decoding performance of \name using our custom binary kernels and binary kernels from GemLite \cite{gemlite}, compared with PyTorch BF16 on 1 NVIDIA H100 GPU.
    }
    \label{fig:h100_gemv_gemlite}
    \vspace{-5px}
\end{figure*}

\begin{figure*}[h]
    \centering
     \includegraphics[width=1.0\linewidth]{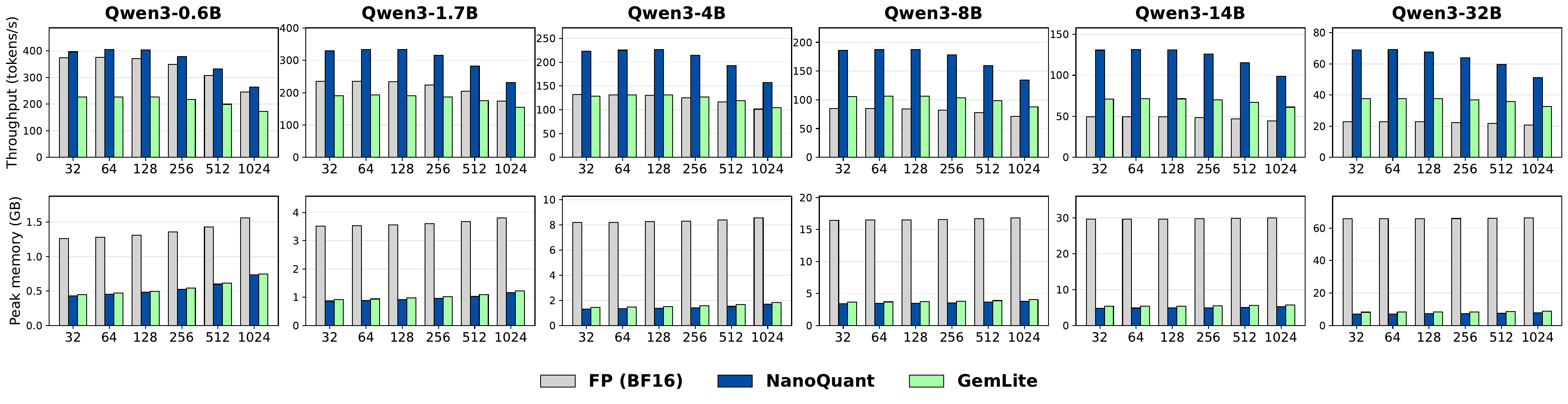}
    \caption{
    LLM decoding performance of \name using our custom binary kernels and binary kernels from GemLite \cite{gemlite}, compared with PyTorch BF16 on 1 NVIDIA A100 GPU.
    }
    \label{fig:a100_gemv_gemlite}
    \vspace{-5px}
\end{figure*}

\newpage

\section{Detailed Model Size Analysis of Binary Weight Quantization Methods}
\label{appendix:model_size}

We analyze the storage cost of binary weight representation methods, encompassing both post-training quantization (PTQ) and quantization-aware training (QAT) via low-rank factorization. We count all stored bits including binarized weights (or binary factor matrices), reconstruction coefficients (typically FP16), and any auxiliary flags or bitmaps required to accurately decode the quantized model. This allows for a fair assessment of the memory requirements across diverse binary compression paradigms.

\subsection{Unified Storage Metric}
\label{sec:bpw_metric}

\paragraph{Bits Per Weight (BPW).}
Let $B_{\mathrm{W_Q}}$ and $B_{\mathrm{W_{FP}}}$ be the total number of bits required to represent quantized weights and full-precision weights, respectively. We derive $\mathrm{BPW}$, the average number of bits per full-precision weight in the quantized model, as:
\begin{equation}
\label{eq:bpw_general}
\mathrm{BPW} = \frac{B_{\mathrm{W_Q}}}{B_{\mathrm{W_{FP}}}} \times 16.
\end{equation}

For instance, if we binarize all weight values in a full-precision matrix $\mathbf{W}_{\text{FP}} \in \mathbb{R}^{n \times m}$ to $\pm1$ without additional metadata, the BPW value would be $\frac{nm}{16nm}\times 16 = 1$.

If we assume $\mathrm{W_Q}$ is a two-dimensional matrix, $B_{\mathrm{W_{FP}}} = 16mn$, and thus $ \mathrm{BPW} = \frac{B_{\mathrm{W_Q}}}{16mn} \times 16 = \frac{\mathrm{B_{W_{Q}}}}{mn} $.

\subsection{Overview of Memory Requirements of Binary Quantization Methods}

\paragraph{Notation.}
We consider a weight matrix $\mathbf{W} \in \mathbb{R}^{n \times m}$, where $n$ is the number of rows and $m$ is the number of columns. We define $k$ as the block size (typically 128), and $c$ as the number of salient columns. The number of blocks per row is denoted as $\lceil m/k \rceil$. We assume reconstruction scales and means are stored in FP16 (16 bits). Additionally, $m$ denotes the storage cost for the salient column bitmap (often compressed).

\paragraph{Methods.}
We analyze the memory requirements of state-of-the-art binary PTQ methods (BiLLM \cite{bi_llm}, ARB-LLM \cite{arb_llm}, STBLLM \cite{stb_llm}, HBLLM \cite{hb_llm}), binary QAT methods using low-rank binary matrices (DBF \cite{dbf}, LittleBit \cite{littlebit}), and \name. 

\subsection{Binary PTQ Methods}

\paragraph{BiLLM.}
\label{sec:bpw_billm}

BiLLM \cite{bi_llm} partitions the weight matrix into salient and non-salient parts. It employs second-order binarization for salient columns and first-order binarization with two quantization groups for non-salient columns. Based on the analysis in \cite{hb_llm}, the total memory requirement $\mathcal{M}_{\textsc{BiLLM}}$ is formulated as:

\begin{align}
\mathcal{M}_{\text{BiLLM}}
=&~ 
\underbrace{2nc + \left\lceil m/k \right\rceil \times 3n \times 16}_{\text{Second-order binarization (Salient)}} \nonumber \\
&+ \underbrace{n(m-c) + \left\lceil m/k \right\rceil \times 2n \times 16 \times 2}_{\text{First-order binarization (Non-salient, 2 groups)}} \nonumber \\
&+ \underbrace{nm}_{\text{Non-salient group bitmap}} + \underbrace{m}_{\text{Salient column bitmap}} \nonumber\\
=&~
n(2m+c) + m + 112n \left\lceil m/k \right\rceil,
\label{eq:mbillm}
\end{align}
where $c$ is the number of salient columns. The term $3n$ in the second-order part accounts for parameters $\alpha_1, \alpha_2$ and the combined mean $\mu$.

With $\mathcal{M}_{\textsc{BiLLM}}$, we can derive the BPW equation as

\begin{align}
\text{BPW}_{\text{BiLLM}}
=&~
\frac{\mathcal{M}_{\textsc{BiLLM}}}{mn}
=
2 + \frac{nc + m + 112n \left\lceil m/k \right\rceil}{mn}.
\end{align}

\paragraph{STBLLM}
\label{sec:bpw_stbllm}

STBLLM \cite{stb_llm} extends the BiLLM framework by introducing $N:M$ sparsity and finer-grained grouping. Unlike BiLLM, which uses 2 groups, STBLLM categorizes non-salient weights into 3 groups (sparse, intermediate, dense) using a trisection search, requiring a 2-bit group bitmap per stored element.

Additionally, STBLLM employs $N:M$ structured sparsity (\eg 4:8 or 6:8) for the non-salient weights. This requires storing the indices of the non-zero elements. For a standard $N:M$ pattern, the index storage $\mathcal{M}_{\text{Indices}}$ is determined by the combinatorics of choosing $N$ positions out of $M$, typically $\lceil \log_2 \binom{M}{N} \rceil$ bits per block of $M$ weights.

The total memory requirement $\mathcal{M}_{\textsc{STBLLM}}$ is formulated as:

\begin{align}
\mathcal{M}_{\textsc{STBLLM}} &= \underbrace{2nc + \lceil m/k \rceil \cdot 3n \cdot 16}_{\text{Second-order binarization (Salient)}} \nonumber \\
&+ \underbrace{\frac{N}{M} \left[ n(m-c) + 2nm \right]}_{\text{Binarized non-zero weights and Group bitmap}} \nonumber \\
&+ \underbrace{\frac{n(m-c)}{M} \cdot \left\lceil \log_2 \binom{M}{N} \right\rceil}_{\text{Sparsity Indices (Metadata)}} \nonumber \\
&+ \underbrace{\lceil m/k \rceil \cdot 2n \cdot 16 \cdot 3}_{\text{FP16 scales/means (3 groups)}} + \underbrace{m}_{\text{Salient column bitmap}}
\label{eq:stbllm_mem_total}
\end{align}

where $n$ and $m$ are the matrix dimensions, $c$ is the number of salient columns, and $k$ is the block size. Dividing by the total original parameters $mn$ yields the Bit-Width Per Weight (BPW) equation:

\begin{align}
\text{BPW}_{\textsc{STBLLM}} &= \frac{\mathcal{M}_{\textsc{STBLLM}}}{mn} \nonumber \\
&= \frac{N}{M} \left( 1 - \frac{c}{m} + 2 \right) + \frac{2c}{m} + \frac{1}{M}\left( 1 - \frac{c}{m} \right) \left\lceil \log_2 \binom{M}{N} \right\rceil \nonumber \\
&+ \frac{144n \lceil m/k \rceil + m}{mn}
\label{eq:stbllm_bpw_final}
\end{align}

\paragraph{ARB-LLM}
\label{sec:bpw_arb}

ARB-LLM \cite{arb_llm} utilizes alternating refined binarization. We analyze the storage for the $\text{ARB-LLM}_{\textsc{RC}}$ variant, as derived in \cite{arb_llm}. This method applies second-order binarization to both salient and non-salient parts using 2 groups:

\begin{align}
\mathcal{M}_{\text{ARBLLM-RC}}
=&~ 
\underbrace{2nc + \left(\left\lceil m/k \right\rceil \times 2n + 2c \right) \times 16}_{\text{Second-order binarization (Salient, 2 groups)}} \nonumber \\
&+ \underbrace{n(m-c) + \left(\left\lceil m/k \right\rceil \times n + (m-c)\right) \times 16 \times 2}_{\text{First-order binarization (Non-salient, 2 groups)}} \nonumber \\
&+ \underbrace{nm}_{\text{Group bitmap}} + \underbrace{m}_{\text{Salient column bitmap}} \nonumber\\
=&~
n(2m+c) + 33m + 64n \left\lceil m/k \right\rceil.
\label{eq:arbllm_rc_mem}
\end{align}

With $\mathcal{M}_{\textsc{ARBLLM-RC}}$, we can derive the BPW equation as 

\begin{align}
\label{eq:arbrc_bpw}
\text{BPW}_{\textsc{ARBLLM-RC}} = \frac{\mathcal{M}_{\textsc{ARBLLM-RC}}}{mn} = 2 + \frac{nc + 33m +  64n \left\lceil m/k \right\rceil}{mn}.
\end{align}

\paragraph{HBLLM}
\label{sec:bpw_hbllm}

HBLLM \cite{hb_llm} introduces structure-aware grouping with two primary variants: HBLLM-row and HBLLM-col.

HBLLM-row employs a neighborhood averaging strategy for non-salient weights and utilizes four subgroups per row for coefficients:
\begin{align}
\mathcal{M}_{\textsc{HBLLM-row}} =&~ 
\underbrace{nm + \left\lceil m/k \right\rceil \times 3n \times 16 \times 2}_{\text{Non-salient weights (2 groups)}} \nonumber \\
&+ \underbrace{nc + \left\lceil m/k \right\rceil \times 2n \times 16 \times 2}_{\text{Salient weights (2 groups)}} \nonumber \\
&+ \underbrace{n(m+c)}_{\text{Group bitmap}} + \underbrace{m}_{\text{Salient column bitmap}} \nonumber\\
=&~
2n(m+c) + m + 160n \left\lceil m/k \right\rceil.
\label{eq:mhbllm_row}
\end{align}

We can derive the BPW equation as

\begin{align}
\label{eq:hbllm_row_bpw}
\text{BPW}_{\text{HBLLM-row}} = \frac{\mathcal{M}_{\textsc{HBLLM-row}}}{mn} = 2 + \frac{2nc + m + 160n \left\lceil m/k \right\rceil}{mn}.
\end{align}

HBLLM-col shares subgroups across two rows and applies intra-band mean sharing, reducing the coefficient overhead:
\begin{align}
\mathcal{M}_{\textsc{HBLLM-col}} =&~ 
\underbrace{n(m-c) + \left\lceil m/k \right\rceil \times 1.5n \times 16 \times 2}_{\text{Non-salient weights (2 groups)}} \nonumber \\
&+ \underbrace{nc + \left\lceil m/k \right\rceil \times 2n \times 16 \times 2}_{\text{Salient weights (2 groups)}} \nonumber \\
&+ \underbrace{nm}_{\text{Group bitmap}} + \underbrace{m}_{\text{Salient column bitmap}} \nonumber\\
&=~
2nm + m + 112n \left\lceil m/k \right\rceil.
\label{eq:mhbllm_col}
\end{align}

We can derive the BPW equation as

\begin{align}
\label{eq:hbllm_col_bpw}
\text{BPW}_{\text{HBLLM-col}} = \frac{\mathcal{M}_{\textsc{HBLLM-col}}}{mn} = 2 + \frac{m + 112n \left\lceil m/k \right\rceil}{mn}.
\end{align}

\subsection{Binary QAT Methods Using Low-Rank Binary Matrices}

\paragraph{DBF and LittleBit}
\label{sec:bpw_dbf_lb}

Double Binary Factorization (DBF) \cite{dbf} and LittleBit \cite{littlebit} approximate the weight matrix $\mathbf{W}$ as:
\begin{equation}
\label{eq:dbf_factorization}
\mathbf{W} \approx \widehat{\mathbf{W}} = \mathrm{Diag}(\mathbf{s}_1)\, \big(\mathbf{U}_{\pm 1}\,\mathrm{Diag}(\mathbf{s}_{\text{mid}})\,\mathbf{V}_{\pm 1}^\top\big)\, \mathrm{Diag}(\mathbf{s}_2),
\end{equation}
where $\mathbf{U}_{\pm1}\in\{\pm1\}^{n\times r}$ and $\mathbf{V}_{\pm1}\in\{\pm1\}^{m\times r}$ are stored as 1-bit entries, and the scales $\mathbf{s}_1\in\mathbb{R}^n$, $\mathbf{s}_{\text{mid}}\in\mathbb{R}^r$, $\mathbf{s}_2\in\mathbb{R}^m$ are stored in FP16. The total storage is:
\begin{equation}
\label{eq:dbf_bits}
\mathcal{M}_{\textsc{DBF}} = r(n+m) + 16(n+r+m).
\end{equation}

We can derive the BPW equation as 

\begin{align}
\label{eq:dbf_col_bpw}
\text{BPW}_{\text{DBF}} = \frac{r(n+m) + 16(n+r+m)}{mn}
\end{align}

\subsection{\name\ (Ours)}
\label{sec:bpw_ours}

Our method simplifies the factorization structure by removing the rank-wise scale $\mathbf{s}_{\text{mid}}$ via a 2-scale system:
\begin{equation}
\label{eq:ours_factorization}
\mathbf{W} \approx \widehat{\mathbf{W}} = \mathrm{Diag}(\mathbf{s}_1)\,\mathbf{U}_{\pm 1}\,\mathbf{V}_{\pm 1}^\top\,\mathrm{Diag}(\mathbf{s}_2).
\end{equation}
The total storage required is:
\begin{equation}
\label{eq:ours_bits}
\mathcal{M}_{\name} = r(n+m) + 16(n+m).
\end{equation}
This reduction in scalar overhead contributes to a lower BPW compared to DBF at the same rank $r$.

We can derive the BPW equation as 

\begin{align}
\label{eq:dbf_col_bpw}
\text{BPW}_{\text{\name}} = \frac{r(n+m) + 16(n+m)}{mn}
\end{align}

\newpage

\subsection{Compression Comparison}
\label{sec:bpw_compare}

\paragraph{Compressed Model Comparison.}
To evaluate the compression capability of each quantization method, we compute the $\mathrm{BPW}_{\text{model}}$ for a model containing $L$ linear layers $\{\mathbf{W}_\ell\}_{l=1}^{L}$ in LLM decoder blocks, with dimensions $n_\ell \times m_\ell$. The total memory bits are given by $\mathcal{M}_{\text{total}} = \sum_{\ell=1}^{L} \mathcal{M}_\ell$, where $\mathcal{M}_\ell$ is calculated using the formulas of each respective method. The effective bits per weight is:

\begin{equation}
\label{eq:total_bits}
\mathrm{BPW}_{\text{model}} = \frac{\sum_{\ell=1}^{L} \mathcal{M}_\ell}{\sum_{\ell=1}^{L} n_\ell m_\ell}.
\end{equation}

Notably, all open-source implementations of the baseline binary PTQ methods have a maximum value of 50 salient columns ($c \leq 50$) and a unified block size value of $k=128$. With these constraints, we can derive the theoretical lower and upper bounds of the compression rate of all baselines.

\begin{table}[htbp]
    \centering
    \caption{Upper and lower bounds of quantized model size (GB) for 1-bit \name and various binary post-training quantization baseline methods, represented as (min, max).}
    \label{tab:model-size}
    \small
    \setlength{\tabcolsep}{3pt} 
    \resizebox{0.95\linewidth}{!}{
        \begin{tabular}{l c c c c c c c c}
            \toprule
            \textbf{Model} & \textbf{BF16} & \textbf{\name} & \textbf{BiLLM} & \textbf{STBLLM$_{4:8}$} & \textbf{STBLLM$_{6:8}$} & \textbf{STBLLM$_{8:8}$} & \textbf{ARB-LLM$_{\text{RC}}$} & \textbf{HBLLM$_{\text{R}}$} \\
            \midrule
            L2-7 & 13.48 & 1.33 & (2.85, 2.86) & (3.36, 3.36) & (3.76, 3.77) & (3.86, 3.87) & (2.55, 2.56) & (3.16, 3.17) \\
            L2-13 & 26.03 & 2.24 & (5.22, 5.23) & (6.21, 6.22) & (7.00, 7.01) & (7.20, 7.21) & (4.63, 4.64) & (5.81, 5.84) \\
            L2-70 & 137.95 & 9.58 & (25.65, 25.69) & (31.00, 31.03) & (35.28, 35.30) & (36.35, 36.39) & (22.47, 22.51) & (28.86, 28.94) \\
            \midrule
            L3-1 & 2.47 & 0.65 & (1.40, 1.40) & (1.48, 1.48) & (1.54, 1.54) & (1.55, 1.56) & (1.36, 1.36) & (1.45, 1.45) \\
            L3-3 & 6.43 & 1.14 & (2.59, 2.59) & (2.81, 2.81) & (2.99, 2.99) & (3.03, 3.04) & (2.46, 2.47) & (2.72, 2.73) \\
            L3-8 & 16.06 & 2.97 & (4.61, 4.62) & (5.16, 5.16) & (5.59, 5.60) & (5.70, 5.71) & (4.29, 4.30) & (4.94, 4.96) \\
            L3-70 & 141.11 & 12.73 & (28.81, 28.85) & (34.15, 34.18) & (38.43, 38.46) & (39.50, 39.54) & (25.62, 25.66) & (32.01, 32.10) \\
            \midrule
            G3-1 & 2.00 & 0.69 & (1.46, 1.46) & (1.51, 1.52) & (1.56, 1.56) & (1.57, 1.57) & (1.43, 1.43) & (1.49, 1.50) \\
            G3-4 & 7.76 & 1.74 & (3.84, 3.85) & (4.09, 4.09) & (4.29, 4.29) & (4.34, 4.35) & (3.69, 3.70) & (3.99, 4.00) \\
            G3-12 & 23.53 & 3.35 & (7.90, 7.91) & (8.74, 8.75) & (9.41, 9.42) & (9.58, 9.59) & (7.40, 7.41) & (8.40, 8.43) \\
            G3-27 & 54.02 & 6.00 & (14.84, 14.87) & (16.84, 16.86) & (18.44, 18.46) & (18.84, 18.87) & (13.65, 13.68) & (16.04, 16.09) \\
            \midrule
            Q3-0.6 & 1.19 & 0.37 & (0.78, 0.78) & (0.82, 0.82) & (0.84, 0.84) & (0.85, 0.85) & (0.76, 0.76) & (0.80, 0.81) \\
            Q3-1.7 & 3.44 & 0.76 & (1.75, 1.76) & (1.86, 1.86) & (1.95, 1.95) & (1.97, 1.98) & (1.69, 1.69) & (1.82, 1.83) \\
            Q3-4 & 8.04 & 1.23 & (2.86, 2.87) & (3.15, 3.15) & (3.37, 3.38) & (3.43, 3.44) & (2.70, 2.70) & (3.03, 3.05) \\
            Q3-8 & 16.38 & 3.35 & (4.99, 5.00) & (5.53, 5.53) & (5.96, 5.97) & (6.07, 6.08) & (4.67, 4.68) & (5.31, 5.33) \\
            Q3-14 & 29.54 & 4.76 & (7.86, 7.87) & (8.89, 8.90) & (9.72, 9.73) & (9.93, 9.94) & (7.25, 7.26) & (8.48, 8.51) \\
            \bottomrule
        \end{tabular}
    }
\end{table}

\begin{table}[htbp]
    \centering
    \caption{Upper and lower bounds of bits-per-weight (BPW) for quantized models of various binary post-training quantization baseline methods, represented as (min, max).}
    \label{tab:model-size}
    \small
    \setlength{\tabcolsep}{3pt} 
    \resizebox{0.95\linewidth}{!}{
        \begin{tabular}{l c c c c c c c c}
            \toprule
            \textbf{Model} & \textbf{BF16} & \textbf{\name} & \textbf{BiLLM} & \textbf{STBLLM$_{4:8}$} & \textbf{STBLLM$_{6:8}$} & \textbf{STBLLM$_{8:8}$} & \textbf{ARB-LLM$_{\text{RC}}$} & \textbf{HBLLM$_{\text{R}}$} \\
            \midrule
            L2-7 & 16.00 & 1.00 & (2.88, 2.89) & (3.50, 3.51) & (4.00, 4.01) & (4.13, 4.14) & (2.51, 2.52) & (3.25, 3.27) \\
            L2-13 & 16.00 & 1.00 & (2.88, 2.88) & (3.50, 3.51) & (4.00, 4.01) & (4.13, 4.13) & (2.51, 2.51) & (3.25, 3.27) \\
            L2-70 & 16.00 & 1.00 & (2.88, 2.88) & (3.50, 3.50) & (4.00, 4.00) & (4.13, 4.13) & (2.50, 2.51) & (3.25, 3.26) \\
            \midrule
            L3-1 & 16.00 & 1.00 & (2.88, 2.90) & (3.50, 3.51) & (4.00, 4.01) & (4.13, 4.15) & (2.51, 2.53) & (3.25, 3.29) \\
            L3-3 & 16.00 & 1.00 & (2.88, 2.89) & (3.50, 3.51) & (4.00, 4.01) & (4.13, 4.14) & (2.51, 2.52) & (3.25, 3.28) \\
            L3-8 & 16.00 & 1.00 & (2.88, 2.89) & (3.50, 3.51) & (4.00, 4.01) & (4.13, 4.14) & (2.51, 2.52) & (3.25, 3.27) \\
            L3-70 & 16.00 & 1.00 & (2.88, 2.88) & (3.50, 3.50) & (4.00, 4.00) & (4.13, 4.13) & (2.50, 2.51) & (3.25, 3.26) \\
            \midrule
            G3-1 & 16.00 & 1.00 & (2.88, 2.91) & (3.50, 3.52) & (4.00, 4.02) & (4.13, 4.16) & (2.52, 2.55) & (3.25, 3.32) \\
            G3-4 & 16.00 & 1.00 & (2.88, 2.89) & (3.50, 3.51) & (4.00, 4.01) & (4.13, 4.14) & (2.51, 2.53) & (3.25, 3.28) \\
            G3-12 & 16.00 & 1.00 & (2.88, 2.89) & (3.50, 3.51) & (4.00, 4.01) & (4.13, 4.14) & (2.51, 2.52) & (3.25, 3.27) \\
            G3-27 & 16.00 & 1.00 & (2.88, 2.88) & (3.50, 3.51) & (4.00, 4.01) & (4.13, 4.13) & (2.50, 2.51) & (3.25, 3.27) \\
            \midrule
            Q3-0.6 & 16.00 & 1.00 & (2.88, 2.92) & (3.50, 3.53) & (4.00, 4.03) & (4.13, 4.17) & (2.52, 2.56) & (3.25, 3.33) \\
            Q3-1.7 & 16.00 & 1.00 & (2.88, 2.90) & (3.50, 3.51) & (4.00, 4.01) & (4.13, 4.15) & (2.51, 2.53) & (3.25, 3.29) \\
            Q3-4 & 16.00 & 1.00 & (2.88, 2.89) & (3.50, 3.51) & (4.00, 4.01) & (4.13, 4.14) & (2.51, 2.52) & (3.25, 3.28) \\
            Q3-8 & 16.00 & 1.00 & (2.88, 2.89) & (3.50, 3.51) & (4.00, 4.01) & (4.13, 4.14) & (2.51, 2.52) & (3.25, 3.27) \\
            Q3-14 & 16.00 & 1.00 & (2.88, 2.88) & (3.50, 3.51) & (4.00, 4.01) & (4.13, 4.13) & (2.51, 2.51) & (3.25, 3.27) \\
            \bottomrule
        \end{tabular}
    }
\end{table}

\newpage

\section{LLM Text Generation via Quantized Models}
\begin{table}[H]
\centering
\setlength{\tabcolsep}{4pt}
\renewcommand{\arraystretch}{1.03}
\caption{
Qualitative text generation examples from Llama-2-7B quantized with \name at different effective bit-widths.
All models are evaluated using the same prompt.
}
\label{tab:quantized-generation}
\vspace{-0.5em}
\begin{tabularx}{\linewidth}{@{}p{0.15\linewidth}X@{}}
\toprule
\textbf{Model} & \textbf{Continuation} \\
\midrule

\multicolumn{2}{@{}p{\linewidth}@{}}{
\textbf{Prompt:}
\colorbox{gray!12}{\strut The discovery of artificial intelligence is revolutionary, in that}
}
\\[-0.15em]
\midrule

\textbf{1.00-bit \name} &
it can help solve problems that require a human mind to solve them, or to make a decision.
For example, it can identify and predict human patterns, and identify the appropriate actions
for those who are in critical situations. It can solve problems that are difficult for humans,
such as finding a solution to a puzzle, or solving a problem that requires a human to solve it.
It also can help solve problems using human intelligence, such as making a decision to solve
the problem, solve a problem or evaluate the situation. In addition, artificial intelligence
can help solve problems related to the development of artificial intelligence, such as predicting
\\

\midrule

\textbf{0.80-bit \name} &
it has been shown to be able to solve problems normally understood by humans. Surely, the future
will be the way that our knowledge is improved, and the development of artificial intelligence
will be crucial to future human development. The future is an exciting time for humanity, and our
future will be enhanced by the technologies that are being developed. We will be able to achieve
the goals that our previous generation have set for humanity. In addition to the information that
has been made available in the past, the technology that has been developed is now being used to
help us improve our knowledge of humanity, and to expand our
\\

\midrule

\textbf{0.55-bit \name} &
it opens the way for many new technological advances. With the creation of computer technology
and the development of digital technology, these technologies have resulted in great changes in
the world. With the rapid advance of computer technology, computers are now able to operate on a
much higher level than previously, and computer technology is now in use in every field of science
and technology. Computer technology has also allowed many new technologies and technologies that
are now available to both the industrial and the scientific industries. The development of computer
technology has also enabled many new innovations in technology and technology that are being developed.
These new technologies
\\

\bottomrule
\end{tabularx}
\vspace{-0.75em}
\end{table}

\end{document}